\documentclass{article}

\usepackage{arxiv}

\usepackage[utf8]{inputenc} 
\usepackage[T1]{fontenc}    
\usepackage{hyperref}       
\usepackage{url}            
\usepackage{booktabs}       
\usepackage{amsfonts}       
\usepackage{nicefrac}       
\usepackage{microtype}      
\usepackage{lipsum}

\usepackage{microtype}
\usepackage{graphicx}
\usepackage{subfigure}
\usepackage{booktabs} 
\usepackage{amsmath}
\usepackage{amsthm}
\usepackage{amssymb}
\usepackage{algorithm}
\usepackage{algorithmic}
\usepackage{color}

\newtheorem{theorem}{Theorem}
\newtheorem{definition}{Definition}
\newtheorem{assumption}{Assumption}
\newtheorem{lemma}{Lemma}
\newtheorem{remark}{Remark}

\begin{document}

\title{BROADCAST: Reducing Both Stochastic and Compression Noise to Robustify Communication-Efficient Federated Learning}

\author{
  Heng Zhu$^{1,2}$\\
  \texttt{zh2013@mail.ustc.edu.cn}
   \And
 Qing Ling$^1$\\
 \texttt{lingqing556@mail.sysu.edu.cn}
 \AND
   \\
 $^1$ Sun Yat-sen University \\  
 $^2$ University of Science and Technology of China
}

\date{}
\maketitle

\begin{abstract}
Communication between workers and the master node to collect local stochastic gradients is a key bottleneck in a large-scale federated learning system. Various recent works have proposed to compress the local stochastic gradients to mitigate the communication overhead. However, robustness to malicious attacks is rarely considered in such a setting. In this work, we investigate the problem of Byzantine-robust compressed federated learning, where the attacks from Byzantine workers can be arbitrarily malicious. We theoretically point out that different to the attacks-free compressed stochastic gradient descent (SGD), its vanilla combination with geometric median-based robust aggregation seriously suffers from the compression noise in the presence of Byzantine attacks. In light of this observation, we propose to reduce the compression noise with gradient difference compression so as to improve the Byzantine-robustness. We also observe the impact of the intrinsic stochastic noise caused by selecting random samples, and adopt the stochastic average gradient algorithm (SAGA) to gradually eliminate the inner variations of regular workers. We theoretically prove that the proposed algorithm reaches a neighborhood of the optimal solution at a linear convergence rate, and the asymptotic learning error is in the same order as that of the state-of-the-art uncompressed method. Finally, numerical experiments demonstrate the effectiveness of the proposed method.
\end{abstract}


\section{Introduction}
With the rapid development of intelligent devices, federated learning has been proposed as an effective approach to fusing local data of distributed devices without jeopardizing data privacy. In a federated learning system, local data are kept at the distributed devices (also termed as workers). At each iteration, the workers send local stochastic gradients to a master node, while the master node aggregates the local stochastic gradients to update the trained model \cite{konevcny2016federated,yang2019federated,MAL-083,zhou2018security}. Beyond data privacy, communication efficiency and robustness to various adversarial attacks are also major concerns of federated learning.

Information exchange between the workers and the master node, especially transmitting the local stochastic gradients from the workers to the master node, is a bottleneck of a federated learning system. In particular, when the trained model is high-dimensional, the local stochastic gradients are high-dimensional too and the communication burden is remarkable. To improve the communication efficiency, several popular strategies have been proposed. One such strategy is to reduce the communication frequency by performing multiple rounds of local updates before one round of transmissions \cite{stich2019local,lin2019don,chen2018lag}. Another orthogonal strategy is to reduce the sizes of transmitted messages by compression. Typical compression methods include quantization that uses limited bits to represent real vectors \cite{alistarh2017qsgd,wen2017terngrad,zhang2017zipml}, and sparsification that enforces sparsity of transmitted vectors \cite{stich2018sparsified,wangni2018gradient,konevcny2018randomized}. In this work we focus on compression. At each iteration, the workers compress the local stochastic gradients and send to the master node. Then, the master node aggregates the received compressed local stochastic gradients to obtain a new direction.

In a federated learning system, however, the process of transmitting the compressed local stochastic gradients is vulnerable to adversarial attacks \cite{MAL-083,guerraoui2018hidden,chen2018internet,cao2019distributed,cao2020distributed}. Not all the workers are guaranteed to be reliable and send the true compressed local stochastic gradients. On the contrary, some of them may send faulty messages to bias the aggregation and lead the optimization process to a wrong direction. To characterize the attacks, we consider the Byzantine attacks model where the number and identities of Byzantine workers are unknown to the master node. The Byzantine workers are assumed to be omniscient, can collude with each other, and may send arbitrary malicious messages \cite{yang2020adversary}. To defend against Byzantine attacks, several robust aggregation rules have been proposed to replace the mean aggregation rule in the popular distributed stochastic gradient descent (SGD) algorithm \cite{chen2017distributed,yin2018byzantine,blanchard2017machine}. These approaches are provably able to alleviate the influence of the malicious messages sent by the Byzantine workers on the optimization process.

In this paper, we investigate the problem of Byzantine-robust federated learning with compression, simultaneously considering Byzantine-robustness and communication efficiency. For Byzantine-robust aggregation rules, the noise introduced by compressing the local stochastic gradients significantly weakens their ability to defend against Byzantine attacks. To theoretically justify this claim, we compare the attacks-free compressed stochastic gradient descent (SGD) and its vanilla combination with geometric median-based robust aggregation under Byzantine attacks. We show that even with unbiased compressors, the Byzantine-robust compressed SGD still suffers from the compression noise. This observation illustrates the necessity of reducing the compression noise in the presence of Byzantine attacks. In addition, the stochastic noise caused by selecting random samples to compute the local stochastic gradients also brings difficulties to handling Byzantine attacks \cite{wu2020federated,khanduri2019byzantine,karimireddy2020learning}. To address these issues, we propose a novel algorithm, termed as BROADCAST (Byzantine-RObust Aggregation with gradient Difference Compression And STochastic variance reduction), to reduce both compression and stochastic noise. To be specific, we apply gradient difference compression \cite{mishchenko2019distributed,horvath2019stochastic} to reduce the compression noise, and adopt the stochastic average gradient algorithm (SAGA) \cite{defazio2014saga} to gradually eliminate the inner variations of regular workers.
Our contributions are summarized as follows:
\begin{itemize}
    \item Our work is among the first attempts to jointly consider Byzantine-robust aggregation and compression in federated learning. Compared with gradient norm thresholding \cite{ghosh2021communication} that removes a predefined fraction of compressed messages, our proposed algorithm does not need any prior knowledge about the fraction of Byzantine workers as long as it is smaller than $\frac{1}{2}$.
    \item We theoretically point out that compared with attacks-free compressed SGD, the Byzantine-robust compressed SGD seriously suffers from the compression noise under Byzantine attacks, emphasizing the necessity of reducing compression noise to enhance Byzantine-robustness.
    \item We prove that the proposed algorithm reaches a neighborhood of the optimal solution at a linear convergence rate, and the asymptotic learning error is in the same order as that of the state-of-the-art uncompressed method \cite{wu2020federated}.
\end{itemize}

\subsection{Related Works}

To achieve compression, we can quantize each coordinate of a transmitted vector into few bits \cite{alistarh2017qsgd,wen2017terngrad,zhang2017zipml}, or obtain a sparser vector by letting some elements be zero \cite{stich2018sparsified,wangni2018gradient,konevcny2018randomized}. These compressors, either unbiased or biased, introduce compression noise that affects convergence of underlying algorithms. Error feedback has been applied to reduce the effect of compression noise and ensure convergence, even with biased compressors \cite{stich2018sparsified,wu2018error,karimireddy2019error,tang2019doublesqueeze}. However, the analysis relies on the assumption of bounded stochastic gradients. Free of this assumption, gradient difference compression is also provably able to reduce the compression noise, requiring the use of unbiased compressors \cite{mishchenko2019distributed,horvath2019stochastic,kovalev2020linearly,liu2020double}. Nevertheless, the influence of gradient difference compression on the Byzantine-robustness has not yet been investigated. Our application and analysis of gradient difference compression in Byzantine-robust federated learning are novel.

%
%
%

Most of the existing Byzantine-robust federated learning methods aim to modify the distributed SGD with robust aggregation rules, such as geometric median \cite{chen2017distributed}, coordinate-wise median \cite{yin2018byzantine}, coordinate-wise trimmed mean \cite{yin2018byzantine}, Krum \cite{blanchard2017machine}, Bulyan \cite{guerraoui2018hidden}, etc. When the workers have non-independent and identical distribution (non-i.i.d.) data, \cite{li2019rsa} proposes a robust stochastic aggregation algorithm that forces the regular workers to reach a common solution, and \cite{he2020byzantine} proposes a resampling strategy to reduce the heterogeneity of data distributions at different workers. Byzantine-robustness and privacy preservation are jointly considered in \cite{so2020byzantine,hashemi2021byzantine}.

Variance reduction techniques have been widely used to reduce stochastic noise to accelerate convergence of stochastic algorithms \cite{defazio2014saga,johnson2013accelerating,shalev2013stochastic,yuan2018variance}. In \cite{horvath2019stochastic,kovalev2020linearly}, the combination of variance reduction and gradient difference compression is investigated. Variance reduction is also important to Byzantine-robustness. It is proved in \cite{wu2020federated} that the use of SAGA can fully eliminate the inner variation and improve the ability of tolerating Byzantine attacks. In \cite{khanduri2019byzantine}, the stochastic variance reduced gradient (SVRG) method is combined with robust aggregation to solve distributed non-convex problems. SGD with momentum is considered in \cite{karimireddy2020learning}, also indicating that variance reduction could effectively enhance the performance of defending Byzantine attacks.

For existing Byzantine-robust methods with compression, \cite{bernstein2018signsgd} shows that SignSGD is able to handle a certain class of Byzantine attacks. However, we will show in the numerical experiments that it fails upon several common Byzantine attacks. In \cite{ghosh2021communication}, gradient norm thresholding is used to remove potential malicious messages with compression, where error feedback is applied to reduce the learning error and Gaussian attacks are tested. However, gradient norm thresholding needs to know the fraction of Byzantine workers, or at least a proper estimate, so as to set the fraction of removals. In contrast, our proposed algorithm does not need any prior knowledge about the number of Byzantine workers as long as it is smaller than $\frac{1}{2}$. In addition, gradient norm thresholding can be viewed as a modified mean aggregation rule. Analyzing its error feedback extension is rather straightforward, and relies on the assumption of bounded stochastic gradients. Our analysis considers the combination of geometric median and gradient difference compression, and is hence more challenging. Further, we do not require the assumption of bounded stochastic gradients. We only assume strong convexity, Lipschitz continuous gradients, and bounded variance, which are common in the analysis of first-order stochastic methods.

Orthogonal to compression, another way to improve communication efficiency is to reduce communication frequency in a predefined or adaptive manner, such as in local SGD \cite{stich2019local,lin2019don} or lazily aggregated gradient \cite{chen2018lag}, respectively. The work of \cite{dong2020communication} combines robust stochastic aggregation \cite{li2019rsa} with lazily aggregated gradient \cite{chen2018lag} to achieve Byzantine-robustness and communication efficiency, which is different to our approach.

\section{Problem Formulation}
Consider a distributed federated learning system with one master node and $W$ workers in a set $\mathcal{W}$. Among these workers, $R$ of them are regular and constitute a set $\mathcal{R}$, while the rest $B$ of them are Byzantine and constitute a set $\mathcal{B}$. Note that the identities of regular and Byzantine workers are unknown. The Byzantine workers are assumed to be omniscient and can collude with each other to send arbitrary malicious messages to the master node. The problem of interest is to find an optimal solution to the finite-sum optimization problem
\begin{equation}
   x^* = \arg \mathop {\min }\limits_x f(x): = \frac{1}{R}\sum\limits_{\omega  \in \mathcal{R}} {{f_\omega }(x)},
\label{problem}
\end{equation}
with
\begin{equation}
    {f_\omega}(x): = \frac{1}{J}\sum\limits_{j = 1}^J {{f_{\omega ,j}}(x)}.
\end{equation}
Here $x \in \mathbb{R}^p$ represents the model parameter to be optimized, $f_{\omega,j}(x)$ is the cost function associated with sample $j$ at regular worker $\omega$, and $f_\omega(x)$ is the local cost function of regular worker $\omega$ averaging on $J$ samples. Our goal is to solve (\ref{problem}) in the presence of arbitrary malicious messages sent by Byzantine workers, while guarantee communication efficiency.

\subsection{Byzantine-Robust SGD}

Without considering communication efficiency and when all the workers are regular, a standard approach to solving (\ref{problem}) is the distributed SGD. At iteration $t$, the master node broadcasts the model parameter $x^t$ to all the workers. Each worker $\omega$ randomly selects a sample with index $i_\omega^t$ to compute a local stochastic gradient $\nabla f_{\omega,i_\omega^t}(x^t)$, and sends it to the master node. The master node averages the received stochastic gradients and updates the model parameter as
\begin{equation}
\label{sgd}
    x^{t+1} = x^t - \gamma \cdot \frac{1}{W}\sum\limits_{\omega \in \mathcal{W}} {\nabla f_{\omega,i_\omega^t}(x^t)},
\end{equation}
where $\gamma$ is the step size.

However, the Byzantine workers can send arbitrary malicious messages to bias the optimization process. That is to say, the message sent by worker $\omega$ at iteration $t$ can be defined as
\begin{equation}
\label{gdef}
    v_\omega^t = \left \{
    \begin{array}{cc}
    \nabla f_{\omega,i_\omega^t}(x^t), & \omega \in \mathcal{R},\\
    *,     & \omega \in \mathcal{B},
    \end{array} \right.
\end{equation}
where $*$ represents an arbitrary $p \times 1$ vector. The distributed SGD is vulnerable to such Byzantine attacks. Even there is only one Byzantine worker, the malicious messages can lead the average operation in \eqref{sgd} to yield zero or infinite \cite{yang2020adversary}.

To address the issue, many robust aggregation rules have been proposed, such as geometric median, coordinate-wise median, coordinate-wise trimmed mean, Krum, and Bulyan, to replace the mean aggregation rule in \eqref{sgd} \cite{guerraoui2018hidden,chen2017distributed,yin2018byzantine,blanchard2017machine}. In this paper, we focus on geometric median, but the idea can be also extended to other robust aggregation rules. After receiving the messages from all the workers, the master node calculates the geometric median as
\begin{equation}
    { \mathop{\rm geomed}\limits_{\omega \in \mathcal{W}} \{ v_\omega^t\} } := \arg \mathop {\min}\limits_{v} \sum\limits_{\omega \in \mathcal{W}} \left \| v-v_\omega^t \right \|.
\end{equation}
Then, the update of model parameter has the form of
\begin{equation}
\label{byrd-sgd}
    x^{t+1} = x^t - \gamma \cdot \mathop{\rm geomed}\limits_{\omega \in \mathcal{W}} \{ v_\omega^t\},
\end{equation}
which is termed as Byzantine-robust SGD. When less than half of the workers are Byzantine ($B < \frac{W}{2}$), the geometric median rule enables robustness to arbitrary malicious messages \cite{chen2017distributed}.

Computing the exact geometric median is time-consuming, especially in the high-dimensional case \cite{weiszfeld2009point}. Thus, we often resort to an $\epsilon$-approximate geometric median that satisfies
\begin{equation}
\label{epsilon}
     \sum\limits_{\omega \in \mathcal{W}} \left \| { \mathop{\rm geomed}\limits_{\omega \in \mathcal{W}} \{ v_\omega^t\} } - v_\omega^t \right \| \leq \mathop {\inf}\limits_{v} \sum\limits_{\omega \in \mathcal{W}} \left \| v-v_\omega^t \right \| + \epsilon.
\end{equation}

\subsection{Compressors}
To reduce the communication burden of the federated learning system, one can compress the local stochastic gradients sent by the workers to the master node. Commonly used compressors are either biased or unbiased. In this paper, we focus on unbiased compressors \cite{alistarh2017qsgd,wangni2018gradient,horvath2019stochastic}. Application and analysis of general, possibly biased compressors are discussed in the Appendix E.



\begin{definition}[Unbiased compressor]
\label{def unbiased}
A randomized operator $\mathcal{Q}$: $\mathbb{R}^p \rightarrow \mathbb{R}^p$ is an {\emph{unbiased compressor}} if it satisfies
\begin{align}
\label{unbiased}
    E_{\mathcal{Q}}[\mathcal{Q}(x)] =& x,  \notag \\
    E_{\mathcal{Q}} \left \|\mathcal{Q}(x) - x\right \|^2 \le & \delta\left \|x\right \|^2, \quad \forall x \in \mathbb{R}^p,
\end{align}
where $\delta$ is a non-negative constant.
\end{definition}
Typical unbiased compressors include:
\begin{itemize}
    \item Randomized quantization \cite{alistarh2017qsgd}: For any real number $r \in [a,b]$, there is a probability $\frac{b-r}{b-a}$ to quantize $r$ into $a$, and $\frac{r-a}{b-a}$ to quantize $r$ into $b$.
    \item Rand-$k$ sparsification \cite{wangni2018gradient}: For any $x \in \mathbb{R}^p$, randomly select $k$ elements of $x$ to be scaled by $\frac{p}{k}$, and let the other elements to be zero.
\end{itemize}
Loosely speaking, $\delta$ can be viewed as the compression ratio. When $\delta$ approaches zero, there is little compression.

%

\subsection{Assumptions}

We make the following assumptions in the analysis.

\begin{assumption}[Strong convexity and Lipschitz continuous gradients]
\label{a1}
The cost function $f$ is $\mu$-strong convex and has $L$-Lipschitz continuous gradients, which means for any $x,y \in \mathbb{R}^p$, it holds that
\begin{equation}
    f(x) \geq f(y) + \langle \nabla f(y), x-y \rangle + \frac{\mu}{2}\left \| x-y \right \|,
\end{equation}
and
\begin{equation}
    \left \| \nabla f(x) - \nabla f(y) \right \| \leq L \left \| x-y \right \|.
\end{equation}
\end{assumption}

\begin{assumption}[Bounded outer variation]
\label{a2}
For any $x \in \mathbb{R}^p$, the variation of the local gradients at the regular workers with respect to the global gradient is upper-bounded by
\begin{equation}
    \frac{1}{R}\sum\limits_{\omega \in \mathcal{R}} \left \| \nabla f_\omega(x) -\nabla f(x) \right \|^2 \leq \sigma^2.
\end{equation}
\end{assumption}

\begin{assumption}[Bounded inner variation]
\label{a3}
For every regular worker $\omega \in \mathcal{R}$ and any $x \in \mathbb{R}^p$, the variation of its stochastic gradient with respect to its local gradient is upper-bounded by
\begin{equation}
    E_{i_\omega^t} \left \| \nabla f_{\omega, i_\omega^t}(x) - \nabla f_\omega(x) \right \|^2 \leq \zeta^2, ~ \forall \omega \in \mathcal{R}.
\end{equation}
\end{assumption}

\begin{assumption}[Bounded stochastic gradients]
\label{a4}
For every regular worker $\omega \in \mathcal{R}$ and any $x \in \mathbb{R}^p$, its stochastic gradient is upper-bounded by
\begin{equation}
    E_{i_\omega^t} \left \| \nabla f_{\omega, i_\omega^t}(x) \right \|^2 \leq G^2, ~ \forall \omega \in \mathcal{R}.
\end{equation}
\end{assumption}

Assumption \ref{a1} is standard in convex analysis. Assumptions \ref{a2} and \ref{a3} bound the outer variation that describes the sample heterogeneity among the regular workers, and the inner variation that describes the sample heterogeneity on every regular worker, respectively \cite{tang2018d}. Assumption \ref{a4} is often used to bound the compression noise \cite{stich2018sparsified,karimireddy2019error,tang2019doublesqueeze}. Note that the Byzantine-robust compressed SGD and SAGA, which are discussed in the ensuing sections, both need this assumption. In contrast, our proposed method does not need this assumption.

\section{Compression and Stochastic Noise in Byzantine-Robust Compressed SGD}

In this section, we begin with analyzing the attacks-free compressed SGD and its naive combination with the geometric median-based robust aggregation rule. We theoretically point out that the latter seriously suffers from the compression noise in the presence of Byzantine attacks, emphasizing the necessity of reducing compression noise to enhance Byzantine-robustness. We also show the impact of stochastic noise caused by selecting random samples.

\subsection{Attacks-Free Compressed SGD}

We consider the attacks-free compressed SGD, and show that the compression noise does not affect the asymptotic learning error. Now all workers $\omega \in \mathcal{W}$ are regular, such that $\mathcal{W} = \mathcal{R}$. The compressed SGD is described as follows. At iteration $t$, the master node broadcasts the model parameter $x^t$ to all the workers. Then each worker $\omega \in \mathcal{W}$ randomly selects a sample with index $i_\omega^t$ to compute a local stochastic gradient $\nabla f_{\omega,i_\omega^t}(x^t)$. After that, each worker $\omega \in \mathcal{W}$ sends the compressed message $\mathcal{Q}(\nabla f_{\omega,i_\omega^t}(x^t))$ to the master node. Upon receiving the compressed messages, the master node updates the model parameter as
\begin{equation}
\label{csgd}
    x^{t+1} = x^t - \gamma \cdot \frac{1}{W}\sum\limits_{\omega \in \mathcal{W}} {\mathcal{Q}(\nabla f_{\omega,i_\omega^t}(x^t))}.
\end{equation}

Below we give the convergence analysis of the attacks-free compressed SGD, and the proof is left to Appendix D.

\begin{theorem}[Convergence of attacks-free compressed SGD]
\label{theorem csgd}
Consider the attacks-free compressed SGD update \eqref{csgd} using an unbiased compressor. Under Assumptions \ref{a1}, \ref{a2}, and \ref{a3}, if the step size $\gamma$ satisfies
\begin{equation} \label{eq:new-001}
    \gamma \leq \frac{2}{(\mu + L)(1+\delta)},
\end{equation}
then it holds that
\begin{equation}
    E \left \| x^t - x^* \right \|^2 \leq \left ( 1 - \frac{2\gamma \mu L}{\mu + L} \right )^t \Delta_1 + \Delta_2,
\end{equation}
where
\begin{align}
    & \Delta_1 := \left \| x^0 - x^* \right \|^2 - \frac{\gamma(\mu + L)(1+\delta)(\zeta^2 + \sigma^2)}{2\mu L}, \\
    & \Delta_2 := \frac{1}{\mu L} ( \sigma^2 + \zeta^2 ).
\end{align}
\end{theorem}

Observe that the asymptotic learning error $\Delta_2$ is determined by the outer variation $\sigma^2$ and the inner variation $\zeta^2$, but irrelevant with $G^2$, the bound of stochastic gradients. In Theorem \ref{theorem csgd}, we use a constant step size for the sake of fair comparison with other algorithms analyzed in this paper. If a diminishing step size is applied, the asymptotic learning error can be eliminated.

By the use of an unbiased compressor, the asymptotic learning error is not influenced by compression and Assumption \ref{a4} is not needed in Theorem \ref{theorem csgd}. However, the bound of step size \eqref{eq:new-001} is affected: A larger compression ratio $\delta$ means that we need to use a smaller step size $\gamma$ to guarantee convergence. With particular note, using biased compressors, compressed stochastic algorithms may even diverge \cite{karimireddy2019error}. We call the impact of compression ratio on the algorithms as \textit{compression noise}, caused by using compressed local stochastic gradients. The impact of outer variation $\sigma^2$ and inner variation $\zeta^2$ on the algorithms is termed as \textit{stochastic noise}, caused by randomly selecting samples.

\subsection{Byzantine-Robust Compressed SGD}

For Byzantine-robust and communication-efficient federated learning, we first consider a vanilla approach that combines the distributed SGD with geometric median aggregation and stochastic gradient compression. We then theoretically point out that compression and stochastic noise significantly weakens its ability to tolerate Byzantine attacks.

The Byzantine-robust compressed SGD is similar to the attacks-free compressed SGD, except that the aggregation rule at the master node is changed from mean to geometric median. At iteration $t$, the master node broadcasts the model parameter $x^t$ to all the workers. Then each regular worker $\omega \in \mathcal{R}$ randomly selects a sample with index $i_\omega^t$ to compute a local stochastic gradient $\nabla f_{\omega,i_\omega^t}(x^t)$. Each Byzantine worker $\omega \in \mathcal{B}$ generates an arbitrary malicious $p \times 1$ vector $*$ instead. We use $v_\omega^t$ given by \eqref{gdef} to denote the vector held by each worker $\omega \in \mathcal{W}$. Different from the Byzantine-robust SGD, now each worker $\omega \in \mathcal{W}$ sends the compressed message $\mathcal{Q}(v_\omega^t)$ to the master node. Upon receiving the compressed messages, the master node updates the model parameter as
\begin{equation}
\label{byrd-csgd}
    x^{t+1} = x^t - \gamma \cdot \mathop{\rm geomed}\limits_{\omega \in \mathcal{W}} \{ \mathcal{Q}(v_\omega^t) \}.
\end{equation}
Here we suppose the Byzantine workers obey the compression rule too. Otherwise, their identities are easy to recognize.

\subsection{Impact of Compression and Stochastic Noise}
For the attacks-free case, directly compressing the stochastic gradients introduces compression noise. For unbiased compressors, it leads to a smaller step size and hence slower convergence as we have shown in Theorem \ref{theorem csgd}. For biased compressors, compression noise may even make the compressed stochastic algorithms divergent \cite{karimireddy2019error}. Due to the randomness introduced in selecting local samples, the stochastic gradients computed in regular workers contain stochastic noise, which may slow down the convergence of stochastic algorithms \cite{johnson2013accelerating}.

In the presence of Byzantine attacks, compression and stochastic noise not only affects convergence, but also significantly influences the effectiveness of robust aggregation rules to defend attacks. The reason is intuitive: Since even the compressed messages sent from the regular workers are noisy, it is difficult to recognize the malicious ones among them. Thus, when the variance of the compressed messages is large, the gap between the average of true stochastic gradients (which is what we want) and the robustly aggregated vector (which is what we have) could be large, too. To justify this intuitive idea and demonstrate the effect of compression and stochastic noise on robust aggregation, we give the following property of geometric median; the proof is delegated to Appendix \ref{app-B}.

\begin{lemma}[Geometric median of compressed vectors]
\label{lemma concentration}
Let $\{ z_\omega, \omega \in \mathcal{W} \}$ be a subset of random vectors distributed in a normed vector space and $\mathcal{Q}(\cdot)$ is an unbiased compressor satisfying Definition \ref{def unbiased}. It holds when $B < \frac{W}{2}$ that
\begin{align}
\label{concentration}
     E &\left \| {\rm{geomed}} \{ \mathcal{Q}(z_\omega) \} - \bar{z} \right \|^2  \\
     \leq & \frac{2C_\alpha^2}{R} \sum\limits_{\omega \in \mathcal{R}} E \left \| z_\omega - E z_\omega \right \|^2 + \frac{2C_\alpha^2}{R} \sum\limits_{\omega \in \mathcal{R}} \left \| E z_\omega - \bar{z} \right \|^2 \notag \\
     & + \frac{2C_\alpha^2 \delta}{R} \sum\limits_{\omega \in \mathcal{R}} E \left \| z_\omega \right \|^2 + \frac{2\epsilon^2}{(W-2B)^2}, \notag
\end{align}
where $\bar{z}:=\frac{1}{R} \sum\limits_{\omega \in \mathcal{R}} E z_\omega$, $\alpha:= \frac{B}{W}$, and $C_\alpha := \frac{2-2\alpha}{1-2\alpha}$.
\end{lemma}

Let $z_\omega$ represent $v_\omega^t$ at iteration $t$. Lemma \ref{lemma concentration} characterizes the mean-square error of the geometric median relative to the average of true stochastic gradients. The mean-square error is bounded by four terms. The first refers to the sum of inner variations, and the second refers to the outer variation. The two terms represent the stochastic noise inside each regular worker and across all the regular workers, respectively. The third is proportional to the compression-related parameter $\delta$, and becomes large when the compression ratio $\delta$ is high. The last is from the inexact $\epsilon$-approximate geometric median. Lemma \ref{lemma concentration} asserts that the compression and stochastic noise enlarges the gap between the geometric median and the average of true stochastic gradients. For other robust aggregation rules such as Krum and coordinate-wise median, similar results are also attainable. The compression and stochastic noise remains affecting the quality of robust aggregation.

Since geometric median aggregation of the compressed messages $\mathcal{Q}(v_\omega^t)$ yields unsatisfactory output as indicated by Lemma \ref{lemma concentration}, the Byzantine-robust compressed SGD in \eqref{byrd-csgd} performs poorly too. Below we analyze its convergence; the proof is delegated to Appendix \ref{app-B}. It converges to a neighborhood of the optimal solution, and the asymptotic learning error is subject to the compression and stochastic noise.

\begin{theorem}[Convergence of Byzantine-robust compressed SGD]
\label{theorem sgd}
Consider the Byzantine-robust compressed SGD update \eqref{byrd-csgd} with $\epsilon$-approximate geometric median aggregation and using an unbiased compressor. Under Assumptions \ref{a1}, \ref{a2}, \ref{a3}, and \ref{a4}, if the number of Byzantine workers satisfies $B < \frac{W}{2}$ and the step size $\gamma$ satisfies
\begin{equation}
    \gamma \leq \frac{\mu}{2L^2},
\end{equation}
then it holds that
\begin{equation}
    E \left \| x^t - x^* \right \|^2 \leq \left ( 1 - \gamma\mu \right )^t \Delta_1 + \Delta_2,
\end{equation}
where
\begin{equation}
    \Delta_1 := \left \| x^0 - x^* \right \|^2 - \Delta_2,
\end{equation}
\begin{align}\label{sgd-error}
    \hspace{-1.4em}\Delta_2 := \frac{2}{\mu^2} \Bigg( 2C_\alpha^2\sigma^2 + 2C_\alpha^2\zeta^2 + 2C_\alpha^2 \delta G^2 + \frac{2\epsilon^2}{(W-2B)^2} \Bigg),
\end{align}
$\alpha:= \frac{B}{W}$, and $C_\alpha := \frac{2-2\alpha}{1-2\alpha}$.
\end{theorem}

Observe that the asymptotic learning error $\Delta_2$ is linear with the inner variation $\zeta^2$, the outer variation $\sigma^2$ and the compression ratio $\delta$. It is the gap introduced by the geometric median, as shown in Lemma \ref{lemma concentration}, that determines the asymptotic learning error. Note that in the analysis of Byzantine-robust compressed SGD, we need Assumption \ref{a4} to bound the compression error. This is different to the analysis of attacks-free compressed SGD in Theorem \ref{theorem csgd} where Assumption \ref{a4} is not required.

Comparing in Theorems \ref{theorem csgd} and \ref{theorem sgd}, we find that under Byzantine attacks, the asymptotic learning error is linear with $C_\alpha^2$, which is determined by the maximum number of Byzantine workers. In addition, the term of $\delta G^2$ appears, showing the larger impact of compression noise in the presence of Byzantine attacks. Technically speaking, this is due to the introduction of the biased robust aggregation to defend against Byzantine attacks. Unlike the unbiased and non-robust mean aggregation, robust aggregation rules are often biased so as to handle the outliers caused by Byzantine attacks. This feature in turn amplifies the impact of compression noise on the asymptotic learning error. Therefore, it is necessary to reduce the compression noise for Byzantine robust compressed federated learning. The stochastic noise, characterized by the inner variation $\zeta^2$ and the outer variation $\sigma^2$, also affects the asymptotic learning error.


Motivated by the analysis of this vanilla approach, we propose to reduce both compression and stochastic noise so as to reach a better neighborhood of the optimal solution.

\section{Reducing Stochastic Noise}
\label{sec:sto}

We start from reducing the impact of stochastic noise. In recent years, variance reduction techniques have been widely used to accelerate convergence of stochastic algorithms \cite{johnson2013accelerating,defazio2014saga}. Motivated by the theoretical findings in Theorem \ref{theorem sgd}, we combine the distributed SAGA, a popular variance reduction approach, to enhance the Byzantine-robustness. We stress that other variance reduction techniques, such as SVRG \cite{khanduri2019byzantine} and momentum \cite{karimireddy2020learning}, could be also applicable.

In the distributed SAGA, each worker stores the most recent stochastic gradient for all of its local data samples. When worker $\omega$ randomly selects a sample with index $i_\omega^t$ at iteration $t$, the corrected stochastic gradient is
\begin{equation}
\nabla f_{\omega,i_\omega^t}(x^t) - \nabla f_{\omega,i_\omega^t}(\phi_{\omega, i_\omega^t}^t) + \frac{1}{J}\sum\limits_{j=1}^J \nabla f_{\omega, j}(\phi_{\omega, j}^t),
\end{equation}
where
\begin{equation}
\label{updatephi}
    \phi_{\omega, j}^{t+1} = \left \{
    \begin{array}{cc}
    \phi_{\omega, j}^t,  & j \neq i_\omega^t,\\
    x^t,  & j=i_\omega^t.
    \end{array} \right.
\end{equation}
That is to say, worker $\omega$ corrects the stochastic gradient by first subtracting the previously stored stochastic gradient of sample $i_\omega^t$, and then adding the average of all the stored stochastic gradients of $J$ samples.

At the presence of Byzantine workers, the vector calculated at $\omega$ can be represented as
\begin{equation}
\label{update g}
    g_\omega^t = \left \{
    \begin{array}{cc}
    \nabla f_{\omega,i_\omega^t}(x^t) - \nabla f_{\omega,i_\omega^t}(\phi_{\omega, i_\omega^t}^t) & \\
    \hspace{3.7em} + \frac{1}{J}\sum\limits_{j=1}^J \nabla f_{\omega, j}(\phi_{\omega, j}^t), & \omega \in \mathcal{R},\\
    *,     & \omega \in \mathcal{B},
    \end{array} \right.
\end{equation}
where $*$ represents an arbitrary $p \times 1$ vector. Every worker $\omega \in \mathcal{W}$ compresses $g_\omega^t$ and sends $\mathcal{Q}(g_\omega^t)$ to the master node. Then the master node performs geometric median aggregation to update the model parameter, as
\begin{equation}
\label{byrd-csaga}
    x^{t+1} = x^t - \gamma \cdot \mathop{\rm geomed}\limits_{\omega \in \mathcal{W}} \{ \mathcal{Q}(g_\omega^t) \}.
\end{equation}
We term it as Byzantine-robust compressed SAGA, whose convergence is stated as follows and the proof is delegated to Appendix \ref{app-C}.


\begin{theorem}[Convergence of Byzantine-robust compressed SAGA]
\label{theorem saga}
Consider the Byzantine-robust compressed SAGA update \eqref{byrd-csaga} with $\epsilon$-approximate geometric median aggregation and using an unbiased compressor. Under Assumptions \ref{a1}, \ref{a2}, and \ref{a4}, if the number of Byzantine workers satisfies $B < \frac{W}{2}$ and the step size $\gamma$ satisfies
\begin{equation}
    \gamma \leq \frac{\mu}{4\sqrt{5}J^2L^2C_\alpha},
\end{equation}
then it holds that
\begin{equation}
    E \left \| x^t - x^* \right \|^2 \leq \left ( 1 - \frac{\gamma\mu}{2} \right )^t \Delta_1 + \Delta_2,
\end{equation}
where
\begin{equation}
    \Delta_1 := \left \| x^0 - x^* \right \|^2 - \Delta_2,
\end{equation}
\begin{equation}\label{saga-error}
    \Delta_2 := \frac{5}{\mu^2} \left ( 2C_\alpha^2\sigma^2 + 2C_\alpha^2 \delta G^2 + \frac{2\epsilon^2}{(W-2B)^2} \right ).
\end{equation}
\end{theorem}

Compared with the asymptotic learning error of Byzantine-robust compressed SGD in \eqref{sgd-error},  in \eqref{saga-error} the inner variation is fully eliminated due to the use of variance reduction.


\begin{remark}
\cite{wu2020federated} proposes the Byzantine-robust SAGA without compression, i.e., each worker $\omega$ directly sends $g_\omega^t$ in \eqref{update g} to the master node. Therein, the asymptotic learning error is $\Delta_2 = \frac{5}{\mu^2} \left ( 2C_\alpha^2\sigma^2 + \frac{2\epsilon^2}{(W-2B)^2} \right )$. One can see that compression is a double-edged sword: It leads to better communication efficiency, but brings higher asymptotic learning error.
\end{remark}

\section{BROADCAST: Reducing Both Compression \& Stochastic Noise}

As mentioned in Section \ref{sec:sto}, directly compressing the corrected stochastic gradients still introduces remarkable compression noise, leading to unsatisfactory Byzantine-robustness. Therefore, we propose to compress the differences between the corrected stochastic gradients and auxiliary vectors to reduce the compression noise. The proposed algorithm, named as BROADCAST (Byzantine-RObust Aggregation with gradient Difference Compression And STochastic variance reduction), jointly reduces the compression and stochastic noise.

\vspace{-0.5em}

\subsection{Gradient Difference Compression}

In gradient difference compression, each worker $\omega$ and the master node maintain the same vector $h_\omega \in \mathbb{R}^p$, which is initialized by the same value and updated following the same rule. At iteration $t$, each worker $\omega$ compresses the difference $g_\omega^t - h_\omega^t$ and sends to the master node. After receiving the compressed difference $\mathcal{Q}(g_\omega^t - h_\omega^t)$, the master node approximates the corrected stochastic gradient as
\begin{equation}
    \hat{g}_\omega^t = h_\omega^t + \mathcal{Q}(g_\omega^t - h_\omega^t).
\end{equation}
Upon collecting all approximations $\hat{g}_\omega^t$, the master node updates the model parameter as
\begin{equation}\label{broadcast-x}
    x^{t+1} = x^t - \gamma \cdot \mathop{{\rm geomed}}\limits_{\omega \in \mathcal{W}} \{ \hat{g}_\omega^t \}.
\end{equation}
With the compressed difference, each worker $\omega$ and the master node both update $h_\omega$ as
\begin{equation}
    h_\omega^{t+1} = h_\omega^t + \beta\mathcal{Q}(g_\omega^t - h_\omega^t),
\end{equation}
where $\beta$ is a hyperparameter. Compared to directly compressing the corrected stochastic gradients, compressing the differences gradually eliminates the compression noise, as we shall see in the theoretical analysis.

\vspace{-0.5em}

\subsection{BROADCAST Algorithm Outline}

BROADCAST is described in Algorithm \ref{algorithmcd}. In each regular worker $\omega$, a stochastic gradient table is kept to store the most recent stochastic gradient for every local sample, and an auxiliary vector $h_\omega^t$ is used to calculate the difference of gradients. Each Byzantine worker $\omega$ may maintain its stochastic gradient table and $h_\omega^t$ for the sake of generating malicious vectors, or not do so but generate malicious vectors in other ways. The master node also maintains $h_\omega^t$ for each worker $\omega$, with the same initialization. At iteration $t$, the master node broadcasts $x^t$ to all the workers. Each regular worker $\omega$ randomly selects a sample and obtains the corrected stochastic gradient $g_\omega^t$ as \eqref{update g}. Next, the difference between $g_\omega^t$ and $h_\omega^t$ is compressed and sent to the master node. Each regular worker $\omega$ then updates $h_\omega^{t+1}$ by adding a scaled compressed difference. The Byzantine workers can generate arbitrary messages but also send the compressed results to cheat the master node. After collecting the compressed differences from all the workers, the master node approximates the corrected stochastic gradients by adding the compressed differences to the stored $h_\omega^t$. Then the master node calculates the geometric median and updates $x$ as \eqref{broadcast-x}. The stored $h_\omega^{t+1}$ in master node is updated in the same way as at each worker $\omega$.

\begin{algorithm}[tb]
\caption{BROADCAST}
\label{algorithmcd}
\textbf{Input}: Step size $\gamma$, hyperparameter $\beta$ \\
\textbf{Initialize}: Initialize $x^0$ for master node and all workers. Initialize $h_\omega^0$ for master node and each worker $\omega$. Initialize $\{ \nabla f_{\omega, j}(\phi_{\omega, j}^0) = \nabla f_{\omega,j}(x^0), j = 1,\dots,J \}$ for each regular worker $\omega$
\begin{algorithmic}[1] 
\FOR {$t=0,1,\dots$}
\STATE \textbf{Master node}:
\STATE Broadcast $x^t$ to all workers
\STATE Receive $\mathcal{Q}(u_\omega^t)$ from all workers
\STATE Obtain approximations $\hat{g}_\omega^t = h_\omega^t + \mathcal{Q}(u_\omega^t)$
\STATE Update $x^{t+1} = x^t - \gamma \cdot \mathop{{\rm geomed}}_{\omega \in \mathcal{W}} \{ \hat{g}_\omega^t \}$
\STATE Update $h_\omega^{t+1} = h_\omega^t + \beta\mathcal{Q}(u_\omega^t)$
\STATE \textbf{Worker $\omega$}:
\IF {$\omega \in \mathcal{R}$}
\STATE Compute $\bar{g}_\omega^t = \frac{1}{J} \sum_{j=1}^J \nabla f_{\omega, j}(\phi_{\omega, j}^t)$
\STATE Randomly sample $i_\omega^t$ from $\{ 1, \dots,J \}$
\STATE Obtain $g_\omega^t = \nabla f_{\omega,i_\omega^t}(x^t) - \nabla f_{\omega,i_\omega^t}(\phi_{\omega, i_\omega^t}^t) + \bar{g}_\omega^t$
\STATE Store $\nabla f_{\omega,i_\omega^t}(\phi_{\omega, i_\omega^t}^t) = \nabla f_{\omega,i_\omega^t}(x^t)$
\STATE Compress $\mathcal{Q}(u_\omega^t) = \mathcal{Q}(g_\omega^t - h_\omega^t)$
\STATE Update $h_\omega^{t+1} = h_\omega^t + \beta\mathcal{Q}(u_\omega^t)$
\STATE Send $\mathcal{Q}(u_\omega^t)$ to master node
\ELSIF {$\omega \in \mathcal{B}$}
\STATE Generate arbitrary malicious vector $g_\omega^t = *$
\STATE Send $\mathcal{Q}(u_\omega^t) = \mathcal{Q}(g_\omega^t)$ to master node
\ENDIF
\ENDFOR
\end{algorithmic}
\end{algorithm}

\subsection{Theoretical Analysis}


Here we establish the convergence of BROADCAST. Note that the technical challenge lies in how to analyze gradient different compression in the biased robust aggregation. We do not have the property that the aggregation of approximated gradient $\{ \hat{g}_\omega^t \}_{\omega \in \mathcal{W}}$ is still unbiased in the previous work.
\begin{theorem}[Convergence of BROADCAST]
\label{theorem cd}
Consider Algorithm \ref{algorithmcd} with $\epsilon$-approximate geometric median aggregation and using an unbiased compressor. Under Assumptions \ref{a1} and \ref{a2}, if the number of Byzantine workers satisfies $B < \frac{W}{2}$ and $\delta C_\alpha^2 \leq \frac{\mu^2}{56L^2}$, the hyperparameter satisfies $\beta(1+\delta) \leq 1$, and the step size $\gamma$ satisfies
\begin{equation}
    \gamma \leq \frac{\beta\mu}{4\sqrt{35}\sqrt{1+5\delta} \cdot J^2L^2C_\alpha},
\end{equation}
then it holds that
\begin{equation}
    E \left \| x^t - x^* \right \|^2 \leq \left ( 1 - \frac{\gamma\mu}{2} \right )^t \Delta_1 + \Delta_2,
\end{equation}
where
\begin{align}
    \Delta_1 :=& \left \| x^0 - x^* \right \|^2 - \Delta_2, \\
    \Delta_2 :=& \frac{70}{17\mu^2} \left ( 2(1 + 6\delta)C_\alpha^2\sigma^2 + \frac{2\epsilon^2}{(W-2B)^2} \right ).
\end{align}
\end{theorem}

The proof of this theorem is in Appendix A. The asymptotic learning error $\Delta_2$ of BROADCAST is no longer dependent on the inner variation and compression noise. The major source of the learning error comes from the outer variation, as well as the computation error in calculating the geometric median. In contrast, the inner variation and compression noise terms both appear in the learning error of the Byzantine-robust compressed SGD, and the compression noise term appears in that of the Byzantine-robust compressed SAGA. Compared to the Byzantine-robust SAGA without compression, the magnitude of learning error is same. If $\delta$ is 0, meaning that no compression is applied, the learning error in BROADCAST is in the same order as that of the Byzantine-robust SAGA without compression. The constant is slightly improved due to proof techniques. Thus, BROADCAST achieves gradient compression for free and achieves the same Byzantine-robustness as its uncompressed counterpart.


\begin{figure*}[htb]
    \centering
    \subfigure[COVTYPE Dataset]{
    \includegraphics[width=0.9\textwidth]{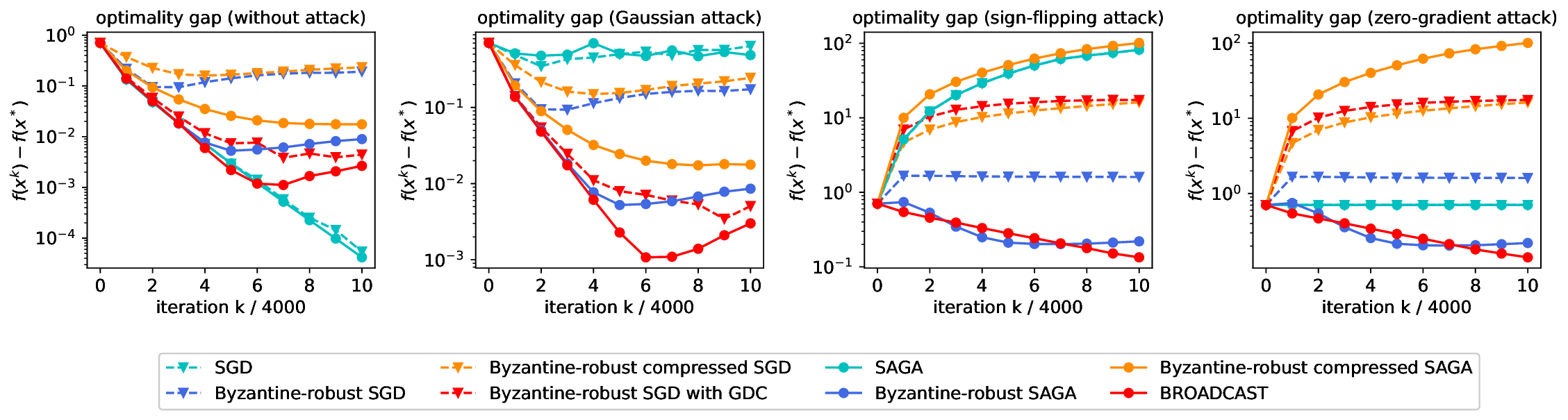}
    }
    \subfigure[Mushrooms Dataset]{
    \includegraphics[width=0.9\textwidth]{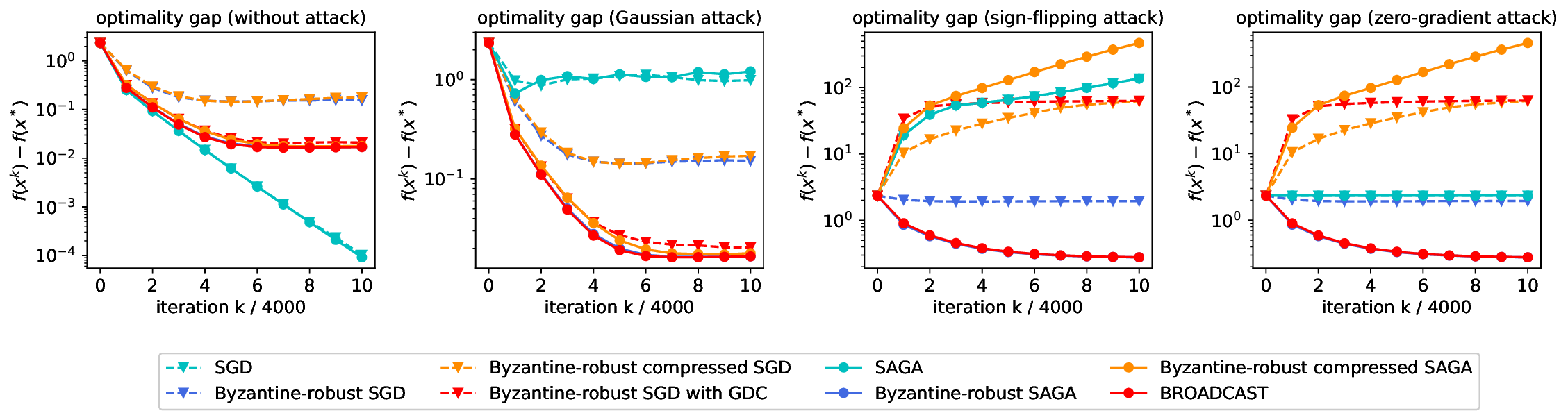}
    }
    \caption{Effect of reducing stochastic and compression noise for logistic regression.}
    \label{compression1}
\end{figure*}

\begin{figure*}[htb]
    \centering
    \includegraphics[width=0.9\textwidth]{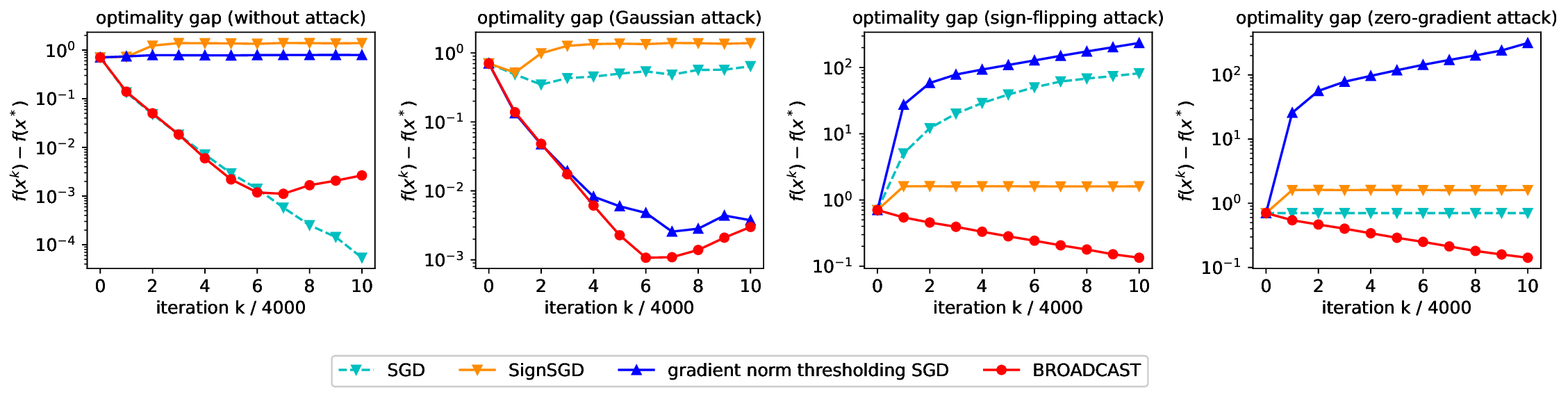}
    \caption{Comparison between proposed algorithm and existing methods for logistic regression.}
    \label{comparison}
\end{figure*}

\begin{figure*}[htb]
    \centering
    \subfigure[COVTYPE Dataset]{
    \includegraphics[width=0.9\textwidth]{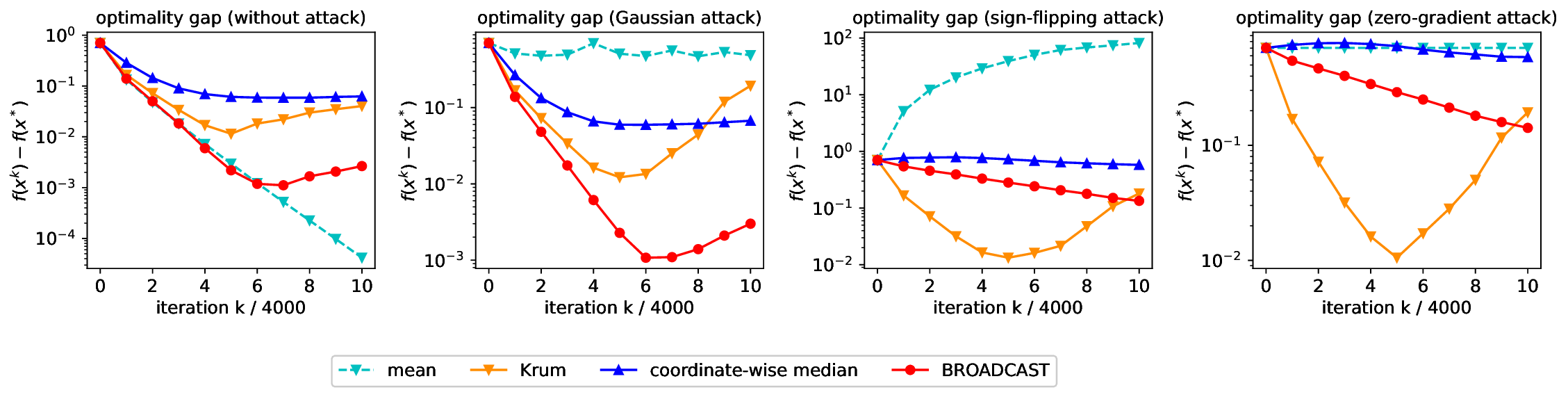}
    }
    \subfigure[Mushrooms Dataset]{
    \includegraphics[width=0.9\textwidth]{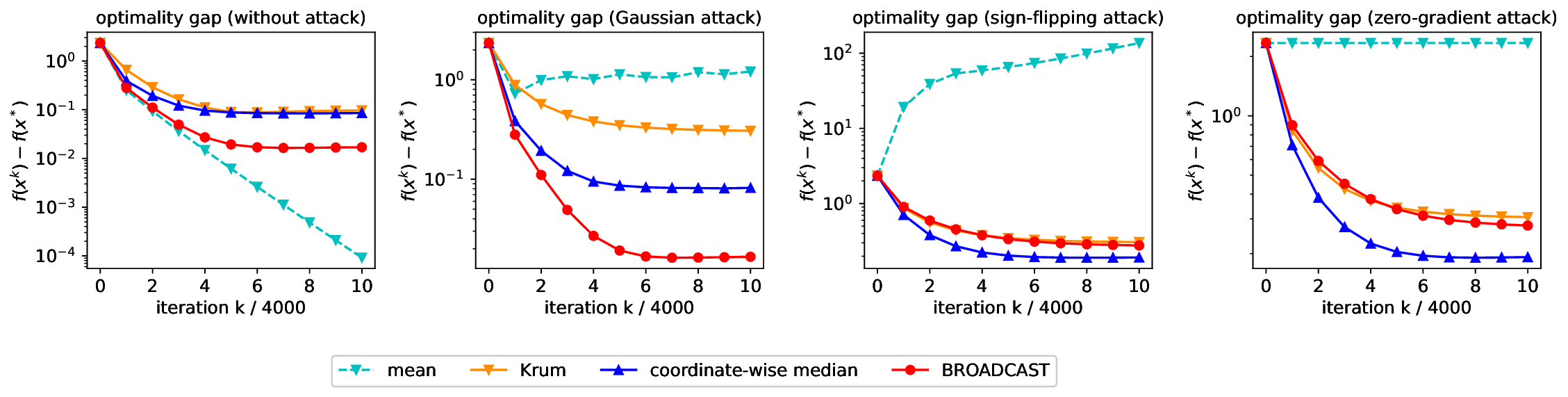}
    }
    \caption{Comparison between different robust aggregation rules for logistic regression.}
    \label{aggregation}
\end{figure*}

\begin{figure*}[htb]
    \centering
    \subfigure[COVTYPE Dataset]{
    \includegraphics[width=0.9\textwidth]{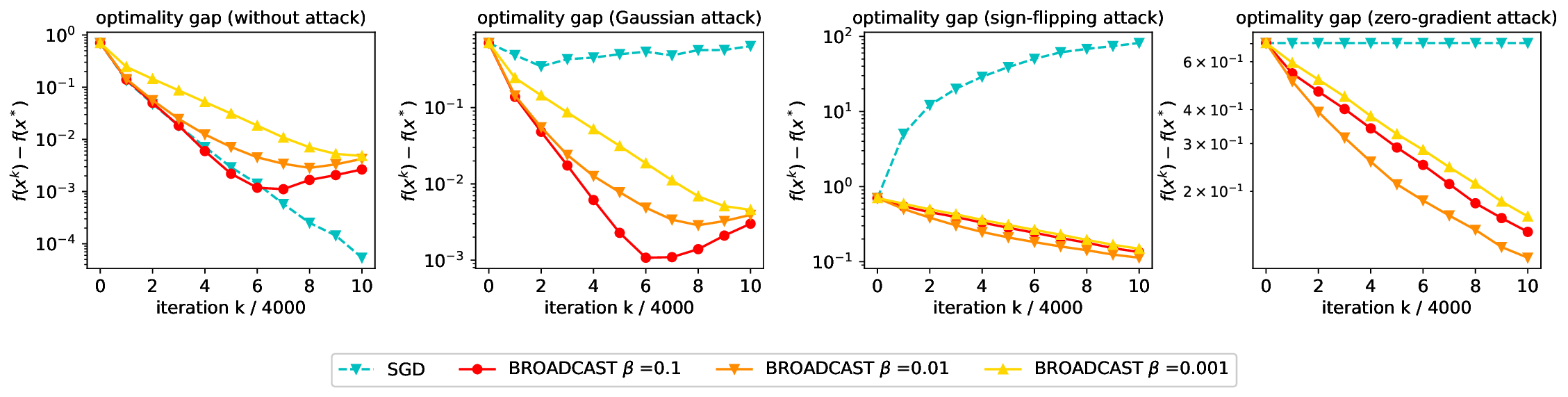}
    }
    \subfigure[Mushrooms Dataset]{
    \includegraphics[width=0.9\textwidth]{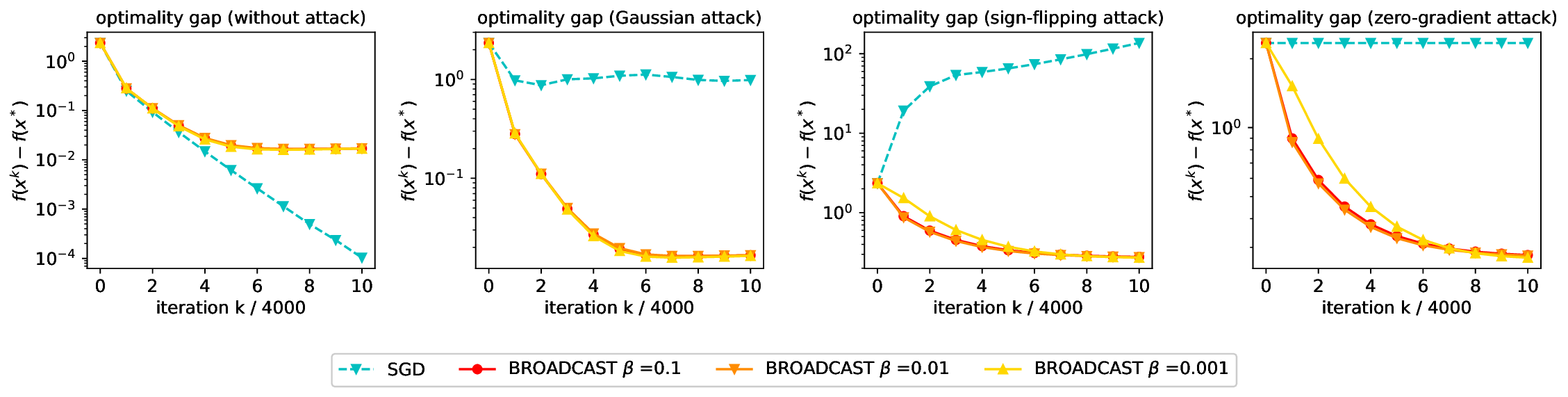}
    }
    \caption{Effect of different $\beta$ in BROADCAST for logistic regression.}
    \label{betapicture}
\end{figure*}

\section{Numerical Experiments}
We present numerical experiments to illustrate the effectiveness of BROADCAST and compare it with existing Byzantine-robust distributed learning algorithms. To validate our theoretical analysis, we consider a strongly convex logistic regression problem on the COVTYPE and Mushroom datasets\footnote{https://www.csie.ntu.edu.tw/\%7ecjlin/libsvmtools/datasets}. We also consider a neural network training task on the MNIST dataset\footnote{http://yann.lecun.com/exdb/mnist} to demonstrate the performance of BROADCAST when the cost function is non-convex. The source code is available at \texttt{https://github.com/oyhah/BROADCAST}.


\subsection{Logistic Regression}
In the logistic regression problem, for sample $j$ at each regular worker $\omega$, the sample cost function is
\vspace{-0.2em}
\begin{equation}
    f_{\omega, j}(x) = \ln (1 + \exp{ (-b_{\omega, j} \langle a_{\omega, j}, x \rangle} )) + \frac{\xi}{2} \left \| x \right \|^2,
\end{equation}
where $a_{\omega, j} \in \mathbb{R}^p$ is the feature vector, $b_{\omega, j} \in \{-1, 1\}$ is the label, and $\xi=0.01$ is the regularization parameter. We use two datasets: the COVTYPE dataset with 581012 samples and $p = 54$ dimensions, as well as the Mushrooms dataset with 8124 samples and $p= 112$ dimensions.

We launch $R = 50$ regular workers and $B = 20$ Byzantine workers. The samples are evenly and randomly allocated to the regular workers. The Byzantine attacks tested here are Gaussian, sign-flipping and zero-gradient. For Gaussian attacks, each Byzantine worker $\omega$ obtains $g_\omega^t$ (or $v_\omega^t$, which we will not distinguish below) from a Gaussian distribution with mean $\frac{1}{R}\sum_{\omega \in \mathcal{R}} g_\omega^t$ and variance 30. For sign-flipping attacks, each Byzantine worker $\omega$ obtains $g_\omega^t$ as $g_\omega^t = u \cdot \frac{1}{R}\sum_{\omega \in \mathcal{R}} g_\omega^t$, where the magnitude is set to $u = -3$. For zero-gradient attacks, each Byzantine worker $\omega$ obtains $g_\omega^t$ as $g_\omega^t = - \frac{1}{B} \sum_{\omega \in \mathcal{R}} g_\omega^t$ so that aggregation at the master node reaches a zero vector in the uncompressed situation. Then Byzantine worker $\omega$ compresses $g_\omega^t$ and sends to the master node. For the compressed methods, the compressor is unbiased rand-$k$ sparsification at the regular agents, and $k/p$ is 0.1. At the Byzantine agents we instead use biased top-$k$ sparsification to guarantee the attacks are strong enough. The hyperparameter $\beta$ in gradient difference compression is 0.1 by default, $\gamma$ is 0.01, and $\epsilon$ is $10^{-5}$.

First, we show the positive effect of reducing compression and stochastic noise to Byzantine-robustness. Fig. \ref{compression1} depicts the optimality gap $f(x^t) - f(x^*)$ of SGD, Byzantine-robust SGD, Byzantine-robust compressed SGD, Byzantine-robust compressed SGD with gradient difference compression (GDC), SAGA, Byzantine-robust SAGA, Byzantine-robust compressed SAGA, and BROADCAST on the two datasets. Observe that SGD and SAGA are unable to tolerate any Byzantine attacks. The Byzantine-robust SAGA significantly outperforms the Byzantine-robust SGD, implying the importance of reducing stochastic noise. With compression, the Byzantine-robust compressed SGD and SAGA are worse than their uncompressed counterparts when facing Gaussian attacks and have large learning errors when facing sign-flipping and zero-gradient attacks. This is due to the accumulation of compression noise, such that the learning errors are dominated by the term $2C_\alpha^2 \delta G^2$ in \eqref{sgd-error} and \eqref{saga-error}, where $G$ is the bound of stochastic gradients and can be large. Our proposed BROADCAST can defend all the three types of Byzantine attacks as the Byzantine-robust SAGA without compression, and slightly outperforms the latter since the Byzantine workers must obey the top-$k$ sparsification rule. Our proposed BROADCAST is also superior to the Byzantine-robust compressed SGD with gradient difference compression, thanks to the reduction of stochastic noise.

Second, Fig. \ref{comparison} compares BROADCAST and the existing Byzantine-robust methods with compression on COVTYPE. SignSGD \cite{bernstein2018signsgd} transmits the signs of stochastic gradients. The gradient norm thresholding SGD \cite{ghosh2021communication} compresses the stochastic gradients and removes a fraction of them with the largest norms before mean aggregation. We let the fraction be $0.3$, which is slightly larger than the exact fraction of Byzantine workers. With the accumulation of compression noise, SignSGD almost fails to defend all attacks and even cannot converge without attacks. For Gaussian attacks, the gradient norm thresholding SGD behaves well because all the malicious messages are removed. But it is unable to remove all the malicious messages under sign-flipping and zero-gradient attacks. In contrast, the proposed BROADCAST performs well in defending various Byzantine attacks.

Third, Fig. \ref{aggregation} compares different robust aggregation rules, all equipped with gradient difference compression and SAGA as in BROADCAST. When there are no attacks, geometric median outperforms Krum and coordinate-wise median. For COVTYPE, under Gaussian attacks, geometric median is the best, while under sign-flipping and zero-gradient attacks, geometric median and Krum perform similarly. For Mushrooms, coordinate-wise median is slightly better under sign-flipping and zero-gradient attacks.

Last, we investigate the effect of hyperparameter $\beta$ in gradient difference compression for the proposed BROADCAST. As shown in Fig. \ref{betapicture}, $\beta$ is chosen as 0.1, 0.01 and 0.001, respectively. According to Theorem \ref{theorem cd}, when $\beta$ is smaller, the step size $\gamma$ should be smaller. Thus we choose $\gamma$ as 0.01 when $\beta$ is 0.1 and 0.01, and as 0.005 when $\beta$ is 0.001. Observe that for different $\beta$, BROADCAST converges to almost the same point and the value of $\beta$ does not influence the performance of BROADCAST in the presence of Byzantine attacks.

\begin{figure*}[htb]
    \centering
    \includegraphics[width=0.9\textwidth]{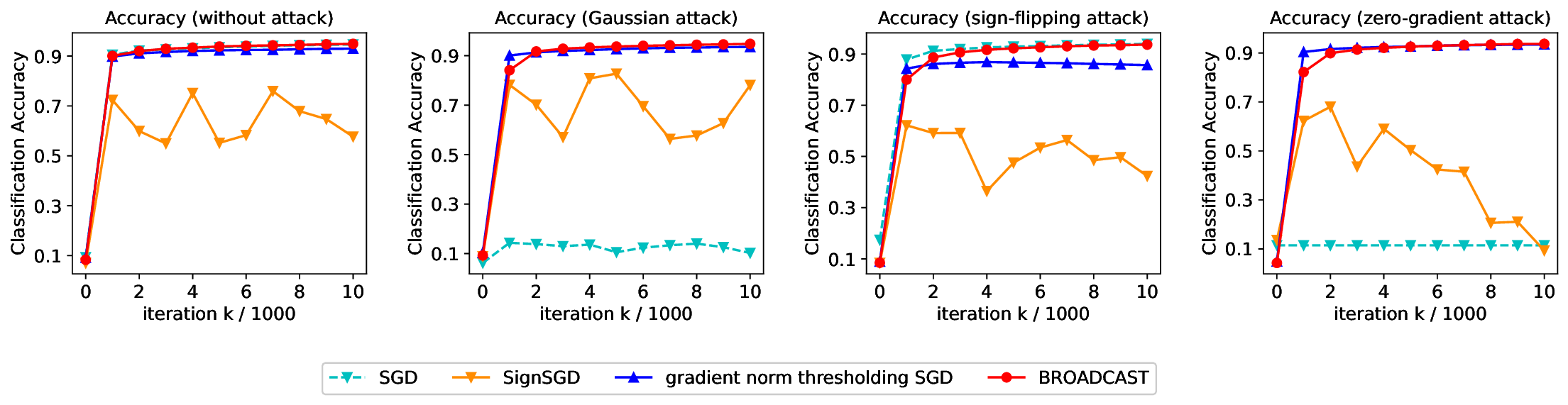}
    \caption{Comparison between proposed algorithm and existing methods for neural network training.}
    \label{comparison_mnist}
\end{figure*}

\subsection{Neural Network Training}

In this experiment, we train a two-layer neural network, each layer containing 50 neurons with tanh activation, for image classification on the MNIST dataset. MNIST is a handwritten digit dataset, containing 60000 training samples and 10000 testing samples, each with 784 dimensions. Here we launch $R=180$ regular workers and $B=20$ Byzantine workers. We apply the BROADCAST algorithm to train the neural network, and compare it with SGD, SignSGD and the gradient norm thresholding SGD. The cost function is cross entropy and the batch size is 5. The compressor is rand-$k$ with $k/p=0.1$.  At the Byzantine agents we instead use biased top-$k$ sparsification to guarantee the attacks are strong enough. The hyperparameter $\beta$ in gradient difference compression is 0.1, $\gamma$ is 0.1, and $\epsilon$ is $10^{-5}$. Since the fraction of Byzantine workers is $10\%$, we reduce the fraction of removed workers in the gradient norm thresholding SGD as $15\%$ for fair comparison. The tested attacks include Gaussian, sign-flipping and zero-gradient, whose settings are the same as those in the logistic regression experiments.



Fig. \ref{comparison_mnist} demonstrates the classification accuracy on the testing samples. Observe that our proposed BROADCAST algorithm reaches high accuracies under all the attacks and outperforms the existing methods. SignSGD behaves unstably. SGD fails under Gaussian and zero-gradient attacks, but works well under sign-flipping attacks that are not strong enough now. Interestingly, under such weak sign-flipping attacks, the gradient norm thresholding SGD has about $8\%$ classification accuracy decline compared to BROADCAST since it is unable to remove all the malicious messages.

\section{Conclusions}

In light of the analysis that a vanilla combination of distributed compressed SGD and geometric median aggregation suffers from compression and stochastic noise in the presence of Byzantine attacks, we develop a novel BROADCAST algorithm to reduce the noise, and consequently, enhance Byzantine-robustness. Theoretical results show that BROADCAST enjoys a linear convergence rate to the neighborhood of the optimal solution and achieves gradient compression for free. Thanks to the successful reduction of both compression and stochastic noise, BROADCAST is demonstrated by numerical experiments to outperform the existing Byzantine-robust methods with compression.

Due to the page limit, now we only analyze the geometric median-based robust aggregation rule and the SAGA-based variance reduction technique. Extending the current results to other robust aggregation rules and variance reduction techniques is natural, and will be an interesting future work. Further improving the communication efficiency by performing multiple
rounds of local updates before one round of transmissions is also of practical importance.



\bibliography{federated}

\begin{thebibliography}{10}

\bibitem{konevcny2016federated}
Jakub Kone{\v{c}}n{\`y}, H~Brendan McMahan, Felix~X Yu, Peter Richt{\'a}rik,
  Ananda~Theertha Suresh, and Dave Bacon,
\newblock ``Federated learning: Strategies for improving communication
  efficiency,''
\newblock {\em arXiv preprint arXiv:1610.05492}, 2016.

\bibitem{yang2019federated}
Qiang Yang, Yang Liu, Tianjian Chen, and Yongxin Tong,
\newblock ``Federated machine learning: Concept and applications,''
\newblock {\em ACM Transactions on Intelligent Systems and Technology}, vol.
  10, no. 2, pp. 1--19, 2019.

\bibitem{MAL-083}
Peter Kairouz and H~Brendan McMahan,
\newblock ``Advances and open problems in federated learning,''
\newblock {\em Foundations and Trends® in Machine Learning}, vol. 14, no. 1,
  pp. 1--210, 2021.

\bibitem{zhou2018security}
Lu~Zhou, Kuo-Hui Yeh, Gerhard Hancke, Zhe Liu, and Chunhua Su,
\newblock ``Security and privacy for the industrial internet of things: An
  overview of approaches to safeguarding endpoints,''
\newblock {\em IEEE Signal Processing Magazine}, vol. 35, no. 5, pp. 76--87,
  2018.

\bibitem{stich2019local}
Sebastian~U Stich,
\newblock ``Local {SGD} converges fast and communicates little,''
\newblock in {\em International Conference on Learning Representations}, 2019.

\bibitem{lin2019don}
Tao Lin, Sebastian~U Stich, Kumar~Kshitij Patel, and Martin Jaggi,
\newblock ``Don't use large mini-batches, use local {SGD},''
\newblock in {\em International Conference on Learning Representations}, 2019.

\bibitem{chen2018lag}
Tianyi Chen, Georgios~B Giannakis, Tao Sun, and Wotao Yin,
\newblock ``{LAG}: Lazily aggregated gradient for communication-efficient
  distributed learning,''
\newblock in {\em Advances in Neural Information Processing Systems}, 2018, pp.
  5050--5060.

\bibitem{alistarh2017qsgd}
Dan Alistarh, Demjan Grubic, Jerry Li, Ryota Tomioka, and Milan Vojnovic,
\newblock ``{QSGD}: Communication-efficient {SGD} via gradient quantization and
  encoding,''
\newblock in {\em Advances in Neural Information Processing Systems}, 2017, pp.
  1709--1720.

\bibitem{wen2017terngrad}
Wei Wen, Cong Xu, Feng Yan, Chunpeng Wu, Yandan Wang, Yiran Chen, and Hai Li,
\newblock ``Terngrad: Ternary gradients to reduce communication in distributed
  deep learning,''
\newblock in {\em Advances in Neural Information Processing Systems}, 2017, pp.
  1509--1519.

\bibitem{zhang2017zipml}
Hantian Zhang, Jerry Li, Kaan Kara, Dan Alistarh, Ji~Liu, and Ce~Zhang,
\newblock ``Zipml: Training linear models with end-to-end low precision, and a
  little bit of deep learning,''
\newblock in {\em International Conference on Machine Learning}, 2017, pp.
  4035--4043.

\bibitem{stich2018sparsified}
Sebastian~U Stich, Jean-Baptiste Cordonnier, and Martin Jaggi,
\newblock ``Sparsified {SGD} with memory,''
\newblock in {\em Advances in Neural Information Processing Systems}, 2018, pp.
  4447--4458.

\bibitem{wangni2018gradient}
Jianqiao Wangni, Jialei Wang, Ji~Liu, and Tong Zhang,
\newblock ``Gradient sparsification for communication-efficient distributed
  optimization,''
\newblock in {\em Advances in Neural Information Processing Systems}, 2018, pp.
  1299--1309.

\bibitem{konevcny2018randomized}
Jakub Kone{\v{c}}n{\`y} and Peter Richt{\'a}rik,
\newblock ``Randomized distributed mean estimation: Accuracy vs.
  communication,''
\newblock {\em Frontiers in Applied Mathematics and Statistics}, vol. 4, no.
  62, 2018.

\bibitem{guerraoui2018hidden}
El~Mahdi~El Mhamdi, Rachid Guerraoui, and S{\'e}bastien Rouault,
\newblock ``The hidden vulnerability of distributed learning in {B}yzantium,''
\newblock in {\em International Conference on Machine Learning}, 2018, pp.
  3521--3530.

\bibitem{chen2018internet}
Yuan Chen, Soummya Kar, and Jose~MF Moura,
\newblock ``The internet of things: Secure distributed inference,''
\newblock {\em IEEE Signal Processing Magazine}, vol. 35, no. 5, pp. 64--75,
  2018.

\bibitem{cao2019distributed}
Xinyang Cao and Lifeng Lai,
\newblock ``Distributed gradient descent algorithm robust to an arbitrary
  number of {B}yzantine attackers,''
\newblock {\em IEEE Transactions on Signal Processing}, vol. 67, no. 22, pp.
  5850--5864, 2019.

\bibitem{cao2020distributed}
Xinyang Cao and Lifeng Lai,
\newblock ``Distributed approximate {Newton's} method robust to byzantine
  attackers,''
\newblock {\em IEEE Transactions on Signal Processing}, vol. 68, pp.
  6011--6025, 2020.

\bibitem{yang2020adversary}
Zhixiong Yang, Arpita Gang, and Waheed~U Bajwa,
\newblock ``Adversary-resilient distributed and decentralized statistical
  inference and machine learning: An overview of recent advances under the
  {B}yzantine threat model,''
\newblock {\em IEEE Signal Processing Magazine}, vol. 37, no. 3, pp. 146--159,
  2020.

\bibitem{chen2017distributed}
Yudong Chen, Lili Su, and Jiaming Xu,
\newblock ``Distributed statistical machine learning in adversarial settings:
  Byzantine gradient descent,''
\newblock {\em Proceedings of the ACM on Measurement and Analysis of Computing
  Systems}, vol. 1, no. 2, pp. 1--25, 2017.

\bibitem{yin2018byzantine}
Dong Yin, Yudong Chen, Ramchandran Kannan, and Peter Bartlett,
\newblock ``Byzantine-robust distributed learning: Towards optimal statistical
  rates,''
\newblock in {\em International Conference on Machine Learning}, 2018, pp.
  5650--5659.

\bibitem{blanchard2017machine}
Peva Blanchard, El~Mahdi El~Mhamdi, Rachid Guerraoui, and Julien Stainer,
\newblock ``Machine learning with adversaries: {B}yzantine tolerant gradient
  descent,''
\newblock in {\em Advances in Neural Information Processing Systems}, 2017, pp.
  119--129.

\bibitem{wu2020federated}
Zhaoxian Wu, Qing Ling, Tianyi Chen, and Georgios~B Giannakis,
\newblock ``Federated variance-reduced stochastic gradient descent with
  robustness to {B}yzantine attacks,''
\newblock {\em IEEE Transactions on Signal Processing}, vol. 68, pp.
  4583--4596, 2020.

\bibitem{khanduri2019byzantine}
Prashant Khanduri, Saikiran Bulusu, Pranay Sharma, and Pramod~K Varshney,
\newblock ``Byzantine resilient non-convex {SVRG} with distributed batch
  gradient computations,''
\newblock {\em arXiv preprint arXiv:1912.04531}, 2019.

\bibitem{karimireddy2020learning}
Sai~Praneeth Karimireddy, Lie He, and Martin Jaggi,
\newblock ``Learning from history for {B}yzantine robust optimization,''
\newblock {\em arXiv preprint arXiv:2012.10333}, 2020.

\bibitem{mishchenko2019distributed}
Konstantin Mishchenko, Eduard Gorbunov, Martin Tak{\'a}{\v{c}}, and Peter
  Richt{\'a}rik,
\newblock ``Distributed learning with compressed gradient differences,''
\newblock {\em arXiv preprint arXiv:1901.09269}, 2019.

\bibitem{horvath2019stochastic}
Samuel Horv{\'a}th, Dmitry Kovalev, Konstantin Mishchenko, Sebastian Stich, and
  Peter Richt{\'a}rik,
\newblock ``Stochastic distributed learning with gradient quantization and
  variance reduction,''
\newblock {\em arXiv preprint arXiv:1904.05115}, 2019.

\bibitem{defazio2014saga}
Aaron Defazio, Francis Bach, and Simon Lacoste-Julien,
\newblock ``{SAGA}: A fast incremental gradient method with support for
  non-strongly convex composite objectives,''
\newblock in {\em Advances in Neural Information Processing Systems}, 2014, pp.
  1646--1654.

\bibitem{ghosh2021communication}
Avishek Ghosh, Raj~Kumar Maity, Swanand Kadhe, Arya Mazumdar, and Kannan
  Ramchandran,
\newblock ``Communication-efficient and {B}yzantine-robust distributed learning
  with error feedback,''
\newblock {\em IEEE Journal on Selected Areas in Information Theory}, vol. 2,
  no. 3, pp. 942--953, 2021.

\bibitem{wu2018error}
Jiaxiang Wu, Weidong Huang, Junzhou Huang, and Tong Zhang,
\newblock ``Error compensated quantized {SGD} and its applications to
  large-scale distributed optimization,''
\newblock {\em arXiv preprint arXiv:1806.08054}, 2018.

\bibitem{karimireddy2019error}
Sai~Praneeth Karimireddy, Quentin Rebjock, Sebastian~U Stich, and Martin Jaggi,
\newblock ``Error feedback fixes {S}ign{SGD} and other gradient compression
  schemes,''
\newblock in {\em International Conference on Machine Learning}, 2019, pp.
  3252--3261.

\bibitem{tang2019doublesqueeze}
Hanlin Tang, Chen Yu, Xiangru Lian, Tong Zhang, and Ji~Liu,
\newblock ``Doublesqueeze: Parallel stochastic gradient descent with
  double-pass error-compensated compression,''
\newblock in {\em International Conference on Machine Learning}, 2019, pp.
  6155--6165.

\bibitem{kovalev2020linearly}
Dmitry Kovalev, Anastasia Koloskova, Martin Jaggi, Peter Richtarik, and
  Sebastian~U Stich,
\newblock ``A linearly convergent algorithm for decentralized optimization:
  Sending less bits for free,''
\newblock {\em arXiv preprint arXiv:2011.01697}, 2020.

\bibitem{liu2020double}
Xiaorui Liu, Yao Li, Jiliang Tang, and Ming Yan,
\newblock ``A double residual compression algorithm for efficient distributed
  learning,''
\newblock in {\em International Conference on Artificial Intelligence and
  Statistics}, 2020, pp. 133--143.

\bibitem{li2019rsa}
Liping Li, Wei Xu, Tianyi Chen, Georgios~B Giannakis, and Qing Ling,
\newblock ``{RSA}: Byzantine-robust stochastic aggregation methods for
  distributed learning from heterogeneous datasets,''
\newblock in {\em AAAI Conference on Artificial Intelligence}, 2019, pp.
  1544--1551.

\bibitem{he2020byzantine}
Lie He, Sai~Praneeth Karimireddy, and Martin Jaggi,
\newblock ``Byzantine-robust learning on heterogeneous datasets via
  resampling,''
\newblock {\em arXiv preprint arXiv:2006.09365}, 2020.

\bibitem{so2020byzantine}
Jinhyun So, Ba{\c{s}}ak G{\"u}ler, and A~Salman Avestimehr,
\newblock ``Byzantine-resilient secure federated learning,''
\newblock {\em IEEE Journal on Selected Areas in Communications}, vol. 39, no.
  7, pp. 2168--2181, 2020.

\bibitem{hashemi2021byzantine}
Hanieh Hashemi, Yongqin Wang, Chuan Guo, and Murali Annavaram,
\newblock ``Byzantine-robust and privacy-preserving framework for fedml,''
\newblock {\em arXiv preprint arXiv:2105.02295}, 2021.

\bibitem{johnson2013accelerating}
Rie Johnson and Tong Zhang,
\newblock ``Accelerating stochastic gradient descent using predictive variance
  reduction,''
\newblock in {\em Advances in Neural Information Processing Systems}, 2013, pp.
  315--323.

\bibitem{shalev2013stochastic}
Shai Shalev-Shwartz and Tong Zhang,
\newblock ``Stochastic dual coordinate ascent methods for regularized loss
  minimization,''
\newblock {\em Journal of Machine Learning Research}, vol. 14, pp. 567--599,
  2013.

\bibitem{yuan2018variance}
Kun Yuan, Bicheng Ying, Jiageng Liu, and Ali~H Sayed,
\newblock ``Variance-reduced stochastic learning by networked agents under
  random reshuffling,''
\newblock {\em IEEE Transactions on Signal Processing}, vol. 67, no. 2, pp.
  351--366, 2018.

\bibitem{bernstein2018signsgd}
Jeremy Bernstein, Jiawei Zhao, Kamyar Azizzadenesheli, and Anima Anandkumar,
\newblock ``{S}ign{SGD} with majority vote is communication efficient and fault
  tolerant,''
\newblock {\em arXiv preprint arXiv:1810.05291}, 2018.

\bibitem{dong2020communication}
Yanjie Dong, Georgios~B Giannakis, Tianyi Chen, Julian Cheng, Md~Hossain, and
  Victor Leung,
\newblock ``Communication-efficient robust federated learning over
  heterogeneous datasets,''
\newblock {\em arXiv preprint arXiv:2006.09992}, 2020.

\bibitem{weiszfeld2009point}
Endre Weiszfeld and Frank Plastria,
\newblock ``On the point for which the sum of the distances to $n$ given points
  is minimum,''
\newblock {\em Annals of Operations Research}, vol. 167, no. 1, pp. 7--41,
  2009.

\bibitem{tang2018d}
Hanlin Tang, Xiangru Lian, Ming Yan, Ce~Zhang, and Ji~Liu,
\newblock ``D2: Decentralized training over decentralized data,''
\newblock in {\em International Conference on Machine Learning}, 2018, pp.
  4848--4856.

\end{thebibliography}
\bibliographystyle{IEEEbib}

\appendix

\section{Analysis of Byzantine-Robust Compressed SGD}
\label{app-B}

Before proving Lemma \ref{lemma concentration} that analyzes the geometric median of compressed random vectors, we review the following lemma that analyzes the geometric median of any random vectors that are not necessarily compressed.

\begin{lemma} \cite{wu2020federated}
\label{lemma geomed}
Let $\{z_\omega, \omega \in \mathcal{W}\}$ be a set of random vectors distributed in a normed vector space. It holds when $B < \frac{W}{2}$ that
\begin{equation}
    E \left \|  \mathop{{\rm geomed}}\limits_{\omega \in \mathcal{R}} \{z_\omega \} \right \|^2 \leq \frac{C_\alpha^2}{R} \sum\limits_{\omega \in \mathcal{R}} E \left \|z_\omega \right \|^2,
\end{equation}
where $\alpha := \frac{B}{W}$ and $C_\alpha := \frac{2-2\alpha}{1-2\alpha}$. Define $z_\epsilon^*$ as an $\epsilon$-approximate geometric median of $\{ z_\omega, \omega \in \mathcal{R} \}$. It holds when $B < \frac{W}{2}$ that
\begin{equation}
    E \left \| z_\epsilon^* \right \|^2 \leq \frac{2C_\alpha^2}{R} \sum\limits_{\omega \in \mathcal{R}} E \left \|z_\omega \right \|^2 + \frac{2\epsilon^2}{(W-2B)^2}.
\end{equation}
\end{lemma}

Now we give the proof of Lemma \ref{lemma concentration}.
\begin{proof}
From Lemma \ref{lemma geomed} and the definition of $\epsilon$-approximate geometric median in \eqref{epsilon} we have
\begin{align}\label{lemma1-1}
      & E \left \| {\rm{geomed}} \{ \mathcal{Q}(z_\omega) \} - \bar{z} \right \|^2  \\
    = & E \left \| {\rm{geomed}} \{ \mathcal{Q}(z_\omega) - \bar{z} \} \right \|^2 \notag \\
    \leq & \frac{2C_\alpha^2}{R} \sum\limits_{\omega \in \mathcal{R}} E \left \| \mathcal{Q}(z_\omega) - \bar{z} \right \|^2 + \frac{2\epsilon^2}{(W-2B)^2}. \notag
\end{align}
Since $\mathcal{Q}(\cdot)$ is an unbiased compressor, we have
\begin{align}\label{lemma1-2}
      & E \left \| \mathcal{Q}(z_\omega) - \bar{z} \right \|^2  \\
    = & E \left \| \mathcal{Q}(z_\omega) - z_\omega + z_\omega - \bar{z} \right \|^2 \notag \\
    = & E \left \| \mathcal{Q}(z_\omega) - z_\omega \right \|^2 + E \left \| z_\omega - \bar{z} \right \|^2 \notag \\
    \leq & \delta E \left \| z_\omega \right \|^2 + E \left \| z_\omega - E z_\omega + E z_\omega - \bar{z} \right \|^2 \notag \\
    = &\delta E \left \| z_\omega \right \|^2 + E \left \| z_\omega - E z_\omega \right \|^2 + \left \| E z_\omega - \bar{z} \right \|^2. \notag
\end{align}
Here the second and the last equalities use $E[\mathcal{Q}(z_\omega) - z_\omega] = 0$ in \eqref{unbiased} and $E[z_\omega-E z_\omega] = 0$, respectively. The inequality comes from $\left \| \mathcal{Q}(z_\omega) - z_\omega \right \|^2 \leq \delta E \left \| z_\omega \right \|^2$ in \eqref{unbiased}. Combining \eqref{lemma1-1} and \eqref{lemma1-2} yields \eqref{concentration} and completes the proof.
\end{proof}

Next we prove Theorem \ref{theorem sgd}.
\begin{proof}
We begin by manipulating $E \left \| x^{t+1} - x^* \right \|^2$ as
\begin{align}
\label{xt+1sgd}
    & E \left \| x^{t+1} - x^* \right \|^2  \\
    = & E \left \| x^t - \gamma \nabla f(x^t) - x^* + x^{t+1} - x^t + \gamma \nabla f(x^t) \right \|^2 \notag \\
    \leq & \frac{1}{1 - \eta} \left \| x^t - \gamma\nabla f(x^t) - x^* \right \|^2  + \frac{1}{\eta} E \left \| x^{t+1} - x^t + \gamma\nabla f(x^t) \right \|^2, \notag
\end{align}
where $0< \eta <1$. The inequality comes from $\left \| a+b \right \|^2 \leq \frac{1}{1-\eta} \left \| a \right \|^2 + \frac{1}{\eta} \left \| b \right \|^2$.

Since $\nabla f(x^*) = 0$, we can bound the first term at the right-hand side of \eqref{xt+1sgd} as
\begin{align}
\label{xtx*sgd}
    & \left \| x^t - \gamma\nabla f(x^t) - x^* \right \|^2  \\
    = & \left \| x^t - \gamma \left ( \nabla f(x^t) - \nabla f(x^*) \right ) - x^* \right \|^2 \notag \\
    = & \left \| x^t -x ^* \right \|^2 - 2\gamma \langle \nabla f(x^t) - \nabla f(x^*), x^t - x^* \rangle \notag \\
    & + \gamma^2 \left \| \nabla f(x^t) - \nabla f(x^*) \right \|^2 \notag \\
    \leq & \left \| x^t - x^* \right \|^2 - 2\gamma\mu \left \| x^t -x^* \right \|^2 + \gamma^2 L^2 \left \| x^t -x^* \right \|^2 \notag \\
    = & (1 - 2\gamma\mu +\gamma^2L^2) \left \| x^t -x^* \right \|^2. \notag
\end{align}
Here we use $L \left \| x^t -x^* \right \| \geq \left \| \nabla f(x^t) - \nabla f(x^*) \right \|$ as $f(x)$ has Lipschitz continuous gradients and $\langle \nabla f(x^t) - \nabla f(x^*), x^t - x^* \rangle \geq \mu \left \| x^t - x^* \right \|^2$ as $f(x)$ is strongly convex.

Substituting (\ref{xtx*sgd}) into (\ref{xt+1sgd}), we can obtain
\begin{align}
\label{xt+1asgd}
    E \left \| x^{t+1} - x^* \right \|^2 \leq \frac{1 - 2\gamma\mu + \gamma^2L^2}{1-\eta} \left \| x^t - x^* \right \|^2  \\
    + \frac{1}{\eta} E \left \| x^{t+1} - x^t + \gamma\nabla f(x^t) \right \|^2. \notag
\end{align}
With $\eta = \frac{\gamma\mu}{2}$, if
\begin{equation}
\label{gamma1sgd}
    \gamma^2L^2 \leq \frac{\gamma\mu}{2},
\end{equation}
it holds that
\begin{equation}
    \frac{1 - 2\gamma\mu + \gamma^2L^2}{1-\eta} \leq 1 - \gamma\mu.
\end{equation}
Therefore, (\ref{xt+1asgd}) can be rewritten as
\begin{align}
\label{xt+1final}
    E \left \| x^{t+1} - x^* \right \|^2 \leq (1 - \gamma\mu) \left \| x^t -x^* \right \|^2 \\
    + \frac{2}{\gamma\mu} E \left \| x^{t+1} - x^t + \gamma\nabla f(x^t) \right \|^2. \notag
\end{align}

Let $z_\epsilon^*$ be an $\epsilon$-approximate geometric median of $\{ \mathcal{Q}(g_\omega^t), \omega \in \mathcal{W} \}$. Based on Lemma \ref{lemma concentration} and Assumptions \ref{a2}, \ref{a3}, and \ref{a4}, the second term at the right-hand side of \eqref{xt+1final} is
\begin{align}
    & E \left \| x^{t+1} - x^t + \gamma\nabla f(x^t) \right \|^2  \\
    = & \gamma^2 E \left \| z_\epsilon^* - \nabla f(x^t) \right \|^2 \notag \\
    \leq & \gamma^2 \{ \frac{2C_\alpha^2}{R} \sum\limits_{\omega \in \mathcal{R}} E \left \| g_\omega^t - \nabla f_\omega(x^t) \right \|^2 + \frac{2\epsilon^2}{(W-2B)^2} \notag \\
    & + \frac{2C_\alpha^2}{R} \sum\limits_{\omega \in \mathcal{R}} \left \| \nabla f_\omega(x^t) - \nabla f(x^t) \right \|^2 + \frac{2C_\alpha^2 \delta}{R} \sum\limits_{\omega \in \mathcal{R}} E \left \| g_\omega^t \right \|^2 \} \notag \\
    \leq & \gamma^2 \left ( 2C_\alpha^2 \zeta^2 + 2C_\alpha^2 \sigma^2 + 2C_\alpha^2 \delta G^2 + \frac{2\epsilon^2}{(W-2B)^2} \right ). \notag
\end{align}
Thus, we have
\begin{align}\label{thm1-1}
    & \hspace{-1em} E \left \| x^{t+1} - x^* \right \|^2 \leq (1 - \gamma\mu) \left \| x^t -x^* \right \|^2  \\
    & \hspace{-1em} + \frac{2\gamma}{\mu} \left ( 2C_\alpha^2 \zeta^2 + 2C_\alpha^2 \sigma^2 + 2C_\alpha^2 \delta G^2 + \frac{2\epsilon^2}{(W-2B)^2} \right ). \notag
\end{align}

Applying telescopic cancellation on \eqref{thm1-1} from iteration $1$ to $t$ yields
\begin{equation}
    E \left \| x^t - x^* \right \|^2 \leq \left ( 1 - \gamma\mu \right )^t \Delta_1 + \Delta_2,
\end{equation}
where
\begin{equation}
    \Delta_1 := \left \| x^0 - x^* \right \|^2 - \Delta_2,
\end{equation}
\begin{small}
\begin{align}
    \Delta_2 := \frac{2}{\mu^2} \Bigg( 2C_\alpha^2\sigma^2 + 2C_\alpha^2\zeta^2 + 2C_\alpha^2 \delta G^2 + \frac{2\epsilon^2}{(W-2B)^2} \Bigg).
\end{align}
\end{small}
\hspace{-1em} This completes the proof.
\end{proof}

\section{Analysis of Byzantine-Robust Compressed SAGA}
\label{app-C}

We first review the following supporting lemma for SAGA.
\begin{lemma} \cite{wu2020federated}
\label{lemma universal}
Under Assumption \ref{a1}, if all regular workers $\omega \in \mathcal{W}$ update $\phi_{\omega, i_\omega^t}$ and $g_\omega^t$ as (\ref{updatephi}) and (\ref{update g}), then the corrected stochastic gradient $g_\omega^t$ satisfies
\begin{equation}
   \hspace{-0.9em} E \left \| g_\omega^t - \nabla f_\omega(x^t) \right \|^2 \leq L^2\frac{1}{J} \sum\limits_{j=1}^J \left \| x^t - \phi_{\omega, j}^t \right \|^2, \forall \omega \in \mathcal{W},
\end{equation}
and
\begin{equation}
    \frac{1}{R} \sum\limits_{\omega \in \mathcal{R}} E \left \| g_\omega^t - \nabla f_\omega(x^t) \right \|^2 \leq L^2S^t,
\end{equation}
where $S^t$ is defined as
\begin{equation}
\label{defs}
    S^t := \frac{1}{R} \sum\limits_{\omega \in \mathcal{R}} \frac{1}{J} \sum\limits_{j=1}^J \left \| x^t - \phi_{\omega, j}^t \right \|^2.
\end{equation}
Further, $S^t$ satisfies that
\begin{align}
\label{St+1}
     E S^{t+1} \leq & 4J E\left \|x^{t+1} - x^t + \gamma \nabla f(x^t) \right \|^2  \\
    &  + 4J\gamma^2L^2 \left \| x^t - x^* \right \|^2 + \left (1 - \frac{1}{J^2} \right )S^t. \notag
\end{align}
\end{lemma}

To handle the bias introduced by geometric median and the effect of compression, the following lemma describes the gap between the $\epsilon$-approximate geometric median of $\{ \mathcal{Q}(g_\omega^t), \omega \in \mathcal{W} \}$ and $\nabla f(x^t)$.

\begin{lemma}
\label{lemma z*saga}
Consider the Byzantine-robust compressed SAGA update \eqref{byrd-csaga} with $\epsilon$-approximate geometric median aggregation and using an unbiased compressor. Under Assumptions \ref{a1}, \ref{a2}, and \ref{a4}, if the number of Byzantine workers satisfies $B < \frac{W}{2}$, then an $\epsilon$-approximate geometric median of $\{ \mathcal{Q}(g_\omega^t), \omega \in \mathcal{W} \}$, denoted by $z_\epsilon^*$, satisfies
\begin{align}\label{lemma4-0000}
     & E \left \| z_\epsilon^* - \nabla f(x^t) \right \|^2 \notag \\
\leq & 2C_\alpha^2L^2S^t + 2C_\alpha^2\sigma^2 + 2C_\alpha^2 \delta G^2 + \frac{2\epsilon^2}{(W-2B)^2}.
\end{align}
\end{lemma}

\begin{proof}
From Lemma \ref{lemma concentration} and Lemma \ref{lemma universal} it holds that
\begin{align}\label{lemma4-0001}
     & E \left \| z_\epsilon^* - \nabla f(x^t) \right \|^2  \\
\leq & 2C_\alpha^2 \frac{1}{R} \sum\limits_{\omega \in \mathcal{R}} E \left \| \mathcal{Q}(g_\omega^t) - \nabla f(x^t) \right \|^2 + \frac{2\epsilon^2}{(W-2B)^2}. \notag
\end{align}
Since $\mathcal{Q}(\cdot)$ is an unbiased compressor, for any $\omega \in \mathcal{R}$ we have
\begin{align}\label{lemma4-0002}
      & E \left \| \mathcal{Q}(g_\omega^t) - \nabla f(x^t) \right \|^2 \\
    = & E \left \| \mathcal{Q}(g_\omega^t) - g_\omega^t + g_\omega^t - \nabla f(x^t) \right \|^2 \notag \\
    = & E \left \| \mathcal{Q}(g_\omega^t) - g_\omega^t \right \|^2 + E \left \| g_\omega^t - \nabla f(x^t) \right \|^2 \notag \\
    \leq & \delta E \left \| g_\omega^t \right \|^2 + E \left \| g_\omega^t - \nabla f_\omega(x^t) + \nabla f_\omega(x^t) - \nabla f(x^t) \right \|^2 \notag \\
    = & \delta E \left \| g_\omega^t \right \|^2 + E \left \| g_\omega^t - \nabla f_\omega(x^t) \right \|^2 + \left \| \nabla f_\omega(x^t) - \nabla f(x^t) \right \|^2. \notag
\end{align}
Here the second and the last equalities use $E[\mathcal{Q}(g_\omega^t) - g_\omega^t] = 0$ in \eqref{unbiased} and $E[g_\omega^t - \nabla f_\omega(x^t)] = 0$ for $\omega \in \mathcal{R}$, respectively. The inequality comes from $E \left \| \mathcal{Q}(g_\omega^t) - g_\omega^t \right \|^2 \leq \delta E \left \| g_\omega^t \right \|^2$ in \eqref{unbiased}. Substituting \eqref{lemma4-0002} into \eqref{lemma4-0001} yields
\begin{align}\label{lemma4-0002000}
 & E \left \| z_\epsilon^* - \nabla f(x^t) \right \|^2 \\
\leq & 2C_\alpha^2 \delta \frac{1}{R} \sum\limits_{\omega \in \mathcal{R}} E \left \| g_\omega^t \right \|^2 + 2C_\alpha^2 \frac{1}{R} \sum\limits_{\omega \in \mathcal{R}} E \left \| g_\omega^t - \nabla f_\omega(x^t) \right \|^2 \notag \\
 & + 2C_\alpha^2 \frac{1}{R} \sum\limits_{\omega \in \mathcal{R}} \left \| \nabla f_\omega(x^t) - \nabla f(x^t) \right \|^2 + \frac{2\epsilon^2}{(W-2B)^2}. \notag
\end{align}
According to Lemma \ref{lemma universal}, as well as Assumptions \ref{a2} and \ref{a4}, the right-hand side of \eqref{lemma4-0002000} can be bounded as in \eqref{lemma4-0000}. 
\end{proof}

Now we can prove Theorem \ref{theorem saga}.
\begin{proof}
When the step size $\gamma$ is sufficiently small such that
\begin{equation}
\label{gamma1saga}
    \gamma^2L^2 \leq \frac{\gamma\mu}{2},
\end{equation}
we can observe that \eqref{xt+1final} also holds true here.

We construct a \emph{Lyapunov function} $T^t$ as
\begin{equation}
\label{defLyasaga}
    T^t := \left \| x^t -x^* \right \|^2 + cS^t,
\end{equation}
where $c$ is any positive constant and $S^t$ is defined as (\ref{defs}). Thus, $T^t$ is non-negative.

Based on (\ref{St+1}), it follows that
\begin{align}
\label{Tt+1asaga}
    & E T^{t+1} = E \left \| x^{t+1}  - x^* \right \|^2 + c E S^{t+1}  \\
    \leq & (1 - \gamma\mu + 4cJ\gamma^2L^2) \left \| x^t - x^* \right \|^2  + \left ( 1 - \frac{1}{J^2} \right )c S^t \notag \\
    & + \left ( \frac{2}{\gamma\mu} + 4cJ \right ) E \left \| x^{t+1} - x^t + \gamma\nabla f(x^t) \right \|^2. \notag
\end{align}
Let $z_\epsilon^*$ be an $\epsilon$-approximate geometric median of $\{ \mathcal{Q}(g_\omega^t), \omega \in \mathcal{W} \}$ and observe the second term at the right-hand side of \eqref{Tt+1asaga}. We have
\begin{align}
\hspace{-1em}    E \left \| x^{t+1} - x^t + \gamma\nabla f(x^t) \right \|^2 = \gamma^2 E \left \| z_\epsilon^* - \nabla f(x^t) \right \|^2.
\end{align}
With this fact and Lemma \ref{lemma z*saga}, \eqref{Tt+1asaga} can be rewritten as
\begin{align}
\label{Tcoeffsaga}
&    E T^{t+1} \leq \left [ 1 - \gamma\mu + 4cJ\gamma^2L^2 \right ] \left \| x^t - x^* \right \|^2 \\
& + \left [ (1 - \frac{1}{J^2})c + \left ( \frac{2}{\gamma\mu} + 4cJ \right )2C_\alpha^2 \gamma^2 L^2 \right ] S^t \notag \\
& + \gamma^2 \left ( \frac{2}{\gamma\mu} + 4cJ \right ) \left ( 2C_\alpha^2\sigma^2 + 2C_\alpha^2 \delta G^2 + \frac{2\epsilon^2}{(W-2B)^2} \right ). \notag
\end{align}

If we let
\begin{equation}
\label{4cJsaga}
    4cJ\gamma^2L^2 \leq \frac{\gamma\mu}{2},
\end{equation}
then it holds
\begin{equation}
    \frac{2}{\gamma\mu} + 4cJ \leq \frac{2}{\gamma\mu} + \frac{\mu}{2\gamma L^2} \leq \frac{5}{2\gamma\mu}.
\end{equation}
Thus, the first coefficient at the right-hand side of \eqref{Tcoeffsaga} is bounded by
\begin{align}
\label{xleqsaga}
    1 - \gamma\mu + 4cJ\gamma^2L^2 \leq 1-\frac{\gamma\mu}{2}.
\end{align}

If $\gamma$ and $c$ are chosen as
\begin{equation}
\label{gamma2saga}
    \frac{\gamma\mu}{2} \leq \frac{1}{2J^2},
\end{equation}
and
\begin{equation}
\label{cvaluesaga}
    c = \frac{10C_\alpha^2J^2\gamma L^2}{\mu} \geq \frac{5}{2} \cdot \frac{2C_\alpha^2J^2\gamma L^2}{\mu(\frac{1}{J^2} - \frac{\gamma\mu}{2})},
\end{equation}
then the second coefficient at the right-hand side of \eqref{Tcoeffsaga} is bounded by
\begin{align}
\label{cleqsaga}
    & (1 - \frac{1}{J^2})c + \left ( \frac{2}{\gamma\mu} + 4cJ \right )2C_\alpha^2 \gamma^2 L^2 \\
    \leq & (1 - \frac{1}{J^2})c + \frac{5}{2} \cdot \frac{2C_\alpha^2\gamma L^2}{\mu} \notag \\
    \leq & (1 - \frac{\gamma\mu}{2})c. \notag
\end{align}

For the last term at the right-hand side of \eqref{Tcoeffsaga}, we have
\begin{align}
\label{deltaleqsaga}
    & \gamma^2 \left ( \frac{2}{\gamma\mu} + 4cJ \right ) \left ( 2C_\alpha^2\sigma^2 + 2C_\alpha^2 \delta G^2 + \frac{2\epsilon^2}{(W-2B)^2} \right )  \\
    & \leq \frac{5\gamma}{2\mu} \left ( 2C_\alpha^2\sigma^2 + 2C_\alpha^2 \delta G^2 + \frac{2\epsilon^2}{(W-2B)^2} \right ). \notag
\end{align}

Therefore, substituting (\ref{xleqsaga}), (\ref{cleqsaga}), and (\ref{deltaleqsaga}), we know that \eqref{Tcoeffsaga} satisfies
\begin{align}
    E T^{t+1} \leq & \left ( 1 - \frac{\gamma\mu}{2} \right ) \left \| x^t - x^* \right \|^2  \\
                   & + \left ( 1 - \frac{\gamma\mu}{2} \right )c S^t + \frac{\gamma \mu}{2}\Delta_2, \notag
\end{align}
where
\begin{equation}
    \Delta_2 := \frac{5}{\mu^2} \left ( 2C_\alpha^2\sigma^2 + 2C_\alpha^2 \delta G^2 + \frac{2\epsilon^2}{(W-2B)^2} \right ).
\end{equation}

Applying telescopic cancellation on \eqref{thm1-1} from iteration $1$ to $t$ yields
\begin{equation}
    E T^t \leq \left ( 1 - \frac{\gamma\mu}{2} \right )^t \left ( T^0 - \Delta_2 \right ) + \Delta_2.
\end{equation}
From the definition of Lyapunov function (\ref{defLyasaga}), we can obtain
\begin{equation}
    E \left \| x^t - x^* \right \|^2 \leq E T^t \leq \left ( 1 - \frac{\gamma\mu}{2} \right )^t \Delta_1 + \Delta_2,
\end{equation}
where
\begin{align}
    \Delta_1 := \left \| x^0 - x^* \right \|^2 - \Delta_2.
\end{align}

Considering the requirements (\ref{gamma1saga}), (\ref{4cJsaga}), and (\ref{gamma2saga}) on the step size, we can choose $\gamma$ as
\begin{equation}
    \gamma \leq \frac{\mu}{4\sqrt{5} \cdot C_\alpha J^2L^2}.
\end{equation}
This completes the proof.
\end{proof}

\section{Analysis of BROADCAST}
\label{app-D}

To handle the bias introduced by geometric median and the effect of gradient difference compression, we give the following lemma to describe the gap between the $\epsilon$-approximate geometric median of $\{ \hat{g}_\omega^t, \omega \in \mathcal{W} \}$ and $\nabla f(x^t)$.
\begin{lemma}
\label{lemma z*cd}
Consider Algorithm \ref{algorithmcd} with $\epsilon$-approximate geometric median aggregation and using an unbiased compressor. Under Assumptions \ref{a1} and \ref{a2}, if the number of Byzantine workers satisfies $B < \frac{W}{2}$ and $\delta C_\alpha^2 \leq \frac{\mu^2}{56L^2}$, then an $\epsilon$-approximate geometric median of $\{ \hat{g}_\omega^t = h_\omega^t + \mathcal{Q}(g_\omega^t - h_\omega^t), \omega \in \mathcal{W} \}$, denoted by $z_\epsilon^*$, satisfies
\begin{align}
\label{eqinlemma}
    & E \left \| z_\epsilon^* - \nabla f(x^t) \right \|^2  \\
    \leq & 2(1+\delta)C_\alpha^2L^2S^t + 4\delta C_\alpha^2H^t + 2C_\alpha^2(1+2\delta)\sigma^2 \notag \\
    & + 4\delta C_\alpha^2L^2 \left \| x^t - x^* \right \|^2 + \frac{2\epsilon^2}{(W-2B)^2}, \notag
\end{align}
where $S^t$ is defined as
\begin{equation}
\label{defs1}
    S^t := \frac{1}{R} \sum\limits_{\omega \in \mathcal{R}} \frac{1}{J} \sum\limits_{j=1}^J \left \| x^t - \phi_{\omega, j}^t \right \|^2.
\end{equation}
and $H^t$ is defined as
\begin{equation}
\label{defH}
    H^t := \frac{1}{R} \sum\limits_{\omega \in \mathcal{R}} E \left \| h_\omega^t \right \|^2.
\end{equation}
\end{lemma}

\begin{proof}
According to Lemma \ref{lemma geomed} in Appendix \ref{app-B}, it holds
\begin{small}
\begin{align}\label{lemma5-0001}
    & E \left \| z_\epsilon^* - \nabla f(x^t) \right \|^2 \\
    \leq& 2C_\alpha^2 \frac{1}{R} \sum\limits_{\omega \in \mathcal{R}} E \left \| \hat{g}_\omega^t - \nabla f(x^t) \right \|^2 + \frac{2\epsilon^2}{(W-2B)^2} \notag \\
    = & 2C_\alpha^2 \frac{1}{R} \sum\limits_{\omega \in \mathcal{R}} E \left \| h_\omega^t + \mathcal{Q}(g_\omega^t - h_\omega^t) - \nabla f(x^t) \right \|^2 + \frac{2\epsilon^2}{(W-2B)^2}. \notag
\end{align}
\end{small}
\hspace{-1em} Consider the first term at the right-hand side of \eqref{lemma5-0001}. Since $\mathcal{Q}(\cdot)$ is an unbiased compressor, for any $\omega \in \mathcal{R}$ we have
\begin{align}
\label{proof znablaf-0001}
      & E \left \| h_\omega^t + \mathcal{Q}(g_\omega^t - h_\omega^t) - \nabla f(x^t) \right \|^2  \\
    = & E \left \| h_\omega^t - g_\omega^t + \mathcal{Q}(g_\omega^t - h_\omega^t) + g_\omega^t - \nabla f(x^t) \right \|^2 \notag \\
    = & E \left \| \mathcal{Q}(g_\omega^t - h_\omega^t) - (g_\omega^t - h_\omega^t) \right \|^2 + E \left \|g_\omega^t - \nabla f(x^t) \right \|^2 \notag \\
    \leq & \delta E \left \| g_\omega^t - h_\omega^t \right \|^2 + E \left \|g_\omega^t - \nabla f(x^t) \right \|^2. \notag
\end{align}
Here the second equality uses $E [ \mathcal{Q}(g_\omega^t - h_\omega^t) - (g_\omega^t - h_\omega^t) ] = 0$ and the inequality uses $E \left \| \mathcal{Q}(g_\omega^t - h_\omega^t) - (g_\omega^t - h_\omega^t) \right \|^2 \leq \delta E \left \| g_\omega^t - h_\omega^t \right \|^2$, both in \eqref{unbiased}. Substituting \eqref{proof znablaf-0001} into \eqref{lemma5-0001} yields
\begin{align}
\label{proof znablaf}
      & E \left \| z_\epsilon^* - \nabla f(x^t) \right \|^2 \leq 2C_\alpha^2 \delta \frac{1}{R} \sum\limits_{\omega \in \mathcal{R}} E \left \| g_\omega^t - h_\omega^t \right \|^2 \\
      & + 2C_\alpha^2 \frac{1}{R} \sum\limits_{\omega \in \mathcal{R}} E \left \|g_\omega^t - \nabla f(x^t) \right \|^2 + \frac{2\epsilon^2}{(W-2B)^2}. \notag
\end{align}
Next, we proceed to bounding the first two terms at the right-hand side of \eqref{proof znablaf}.

For the first term at the right-hand side of \eqref{proof znablaf}, which arises from gradient difference compression, it holds
\begin{align}
\label{proof gh}
    & \frac{1}{R}  \sum\limits_{\omega \in \mathcal{R}} E \left \| g_\omega^t - h_\omega^t \right \|^2  \\
    = & \frac{1}{R} \sum\limits_{\omega \in \mathcal{R}} E \left \| g_\omega^t - \nabla f_\omega(x^t) + \nabla f_\omega(x^t) - h_\omega^t \right \|^2 \notag \\
    = & \frac{1}{R} \sum\limits_{\omega \in \mathcal{R}} E \left \| g_\omega^t - \nabla f_\omega(x^t) \right \|^2 + \frac{1}{R} \sum\limits_{\omega \in \mathcal{R}} E \left \| \nabla f_\omega(x^t) - h_\omega^t \right \|^2 \notag \\
    \leq & L^2S^t + \frac{1}{R} \sum\limits_{\omega \in \mathcal{R}} E \left \| \nabla f_\omega(x^t) - h_\omega^t \right \|^2. \notag
\end{align}
Here the second equality comes from $E [ g_\omega^t - \nabla f_\omega(x^t) ] =0$ and the inequality uses Lemma \ref{lemma universal} in Appendix \ref{app-C}. Further, using $\nabla f(x^*) = 0$, $\left \|a+b \right \|^2 \leq 2 \left \| a \right \|^2 + 2 \left \| b \right \|^2$, and the definition of $H^t$ in (\ref{defH}), from \eqref{proof gh} we have
\begin{align}
     & \frac{1}{R}  \sum\limits_{\omega \in \mathcal{R}} E \left \| g_\omega^t - h_\omega^t \right \|^2  \\
    \leq & L^2S^t + \frac{1}{R} \sum\limits_{\omega \in \mathcal{R}} E \left \| \nabla f_\omega(x^t) - \nabla f(x^*) - h_\omega^t \right \|^2 \notag \\
    \leq & L^2S^t + \frac{1}{R} \sum\limits_{\omega \in \mathcal{R}} 2 \left \| \nabla f_\omega(x^t) - \nabla f(x^*) \right \|^2 + \frac{1}{R} \sum\limits_{\omega \in \mathcal{R}} 2E \left \| h_\omega^t \right \|^2 \notag \\
    = & L^2S^t + 2H^t + \frac{1}{R} \sum\limits_{\omega \in \mathcal{R}} 2 \left \| \nabla f_\omega(x^t) - \nabla f(x^*) \right \|^2. \notag
\end{align}
Since $f(x^t)$ has $L$-Lipschitz continuous gradients according to Assumption \ref{a1} and the outer variation is bounded according to Assumption \ref{a2}, we further obtain
\begin{align}
    & \frac{1}{R}  \sum\limits_{\omega \in \mathcal{R}} \left \| g_\omega^t - h_\omega^t \right \|^2 \leq L^2S^t + 2H^t  \\
    & + \frac{1}{R} \sum\limits_{\omega \in \mathcal{R}} 2 \left \| \nabla f_\omega(x^t) - \nabla f(x^t) + \nabla f(x^t) - \nabla f(x^*) \right \|^2 \notag \\
    = & L^2S^t + 2H^t + \frac{1}{R} \sum\limits_{\omega \in \mathcal{R}} 2 \left \| \nabla f_\omega(x^t) - \nabla f(x^t) \right \|^2 \notag \\
    & + \frac{1}{R} \sum\limits_{\omega \in \mathcal{R}} 2 \left \| \nabla f(x^t) - \nabla f(x^*) \right \|^2 \notag \\
    \leq & L^2S^t + 2H^t + 2\sigma^2 + 2L^2 \left \| x^t - x^* \right \|^2. \notag
\end{align}
Here the equality uses $\sum_{\omega \in \mathcal{R}} ( \nabla f_\omega(x^t) - \nabla f(x^t) ) =0$.

Using Lemma \ref{lemma universal} and Assumption \ref{a2}, the second term at the right-hand side of (\ref{proof znablaf}) can be upper-bounded as
\begin{align}
\label{proof gnablaf}
    & \frac{1}{R} \sum\limits_{\omega \in \mathcal{R}} E \left \|g_\omega^t - \nabla f(x^t) \right \|^2  \\
    = & \frac{1}{R} \sum\limits_{\omega \in \mathcal{R}} E \left \|g_\omega^t - \nabla f_\omega(x^t) + \nabla f_\omega(x^t) - \nabla f(x^t) \right \|^2 \notag \\
    = & \frac{1}{R} \sum\limits_{\omega \in \mathcal{R}} E \left \|g_\omega^t - \nabla f_\omega(x^t) \right \| + \frac{1}{R} \sum\limits_{\omega \in \mathcal{R}} \left \| \nabla f_\omega(x^t) - \nabla f(x^t) \right \|^2 \notag \\
    \leq & L^2S^t + \sigma^2. \notag
\end{align}
Here the equality also uses $\sum_{\omega \in \mathcal{R}} ( \nabla f_\omega(x^t) - \nabla f(x^t) ) =0$.

Substituting (\ref{proof gh}) and (\ref{proof gnablaf}) into (\ref{proof znablaf}) yields \eqref{eqinlemma}.
\end{proof}

Observe that oifferent to Lemma \ref{lemma z*saga} for Byzantine-robust compressed SAGA, a new term $H^t$ appears in Lemma \ref{lemma z*cd}, and also in the ensuing analysis. The following lemma characterizes the evolution of $H^t$.

\begin{lemma}
Consider Algorithm \ref{algorithmcd} with $\epsilon$-approximate geometric median aggregation and using an unbiased compressor. Under Assumptions \ref{a1} and \ref{a2}, if the hyperparameter satisfies $\beta(1+\delta) \leq 1$, then it holds that
\begin{equation}
\label{Ht+1}
 \hspace{-1.0em}   E H^{t+1} \leq (1-\beta)H^t + \beta L^2S^t + \beta\sigma^2 + \beta L^2 \left \| x^t - x^* \right \|^2,
\end{equation}
where $H^t$ is defined as in (\ref{defH}).
\end{lemma}

\begin{proof}
By $E_{\mathcal{Q}} \left \|\mathcal{Q}(x) - x\right \|^2 \le \delta\left \|x\right \|^2$ in (\ref{unbiased}), we obtain
\begin{align}
    E H^{t+1} = & \frac{1}{R} \sum\limits_{\omega \in \mathcal{R}} E \left \| h_\omega^{t+1} \right \|^2  \\
    = & \frac{1}{R} \sum\limits_{\omega \in \mathcal{R}} E \left \| h_\omega^t + \beta \mathcal{Q}(g_\omega^t - h_\omega^t) \right \|^2 \notag \\
    = & \frac{1}{R} \sum\limits_{\omega \in \mathcal{R}} E \left \| h_\omega^t \right \|^2 + \frac{1}{R} \sum\limits_{\omega \in \mathcal{R}} 2\beta E \langle h_\omega^t, g_\omega^t - h_\omega^t \rangle \notag \\
    & + \frac{1}{R} \sum\limits_{\omega \in \mathcal{R}} \beta^2 E \left \| \mathcal{Q}(g_\omega^t - h_\omega^t) \right \|^2 \notag \\
    \leq & \frac{1}{R} \sum\limits_{\omega \in \mathcal{R}} E \left \| h_\omega^t \right \|^2 + \frac{1}{R} \sum\limits_{\omega \in \mathcal{R}} 2\beta E \langle h_\omega^t, g_\omega^t - h_\omega^t \rangle \notag \\
    & + \frac{1}{R} \sum\limits_{\omega \in \mathcal{R}} \beta^2(1+\delta) E \left \| g_\omega^t - h_\omega^t \right \|^2. \notag
\end{align}
With $\beta(1+\delta) \leq 1$, we further have
\begin{align}
    & \hspace{-2em} E H^{t+1} \leq \frac{1}{R} \sum\limits_{\omega \in \mathcal{R}} E \left \| h_\omega^t \right \|^2 + \frac{1}{R} \sum\limits_{\omega \in \mathcal{R}} 2\beta E \langle h_\omega^t, g_\omega^t - h_\omega^t \rangle  \\
    & + \frac{1}{R} \sum\limits_{\omega \in \mathcal{R}} \beta E \left \| g_\omega^t - h_\omega^t \right \|^2 \notag \\
    = & \frac{1}{R} \sum\limits_{\omega \in \mathcal{R}} E \left \| h_\omega^t \right \|^2 + \frac{1}{R} \sum\limits_{\omega \in \mathcal{R}} \beta E \langle g_\omega^t - h_\omega^t, g_\omega^t + h_\omega^t \rangle \notag \\
    = & \frac{1}{R} \sum\limits_{\omega \in \mathcal{R}} E \left \| h_\omega^t \right \|^2 + \frac{1}{R} \sum\limits_{\omega \in \mathcal{R}} \beta E \left \| g_\omega^t \right \|^2 - \frac{1}{R} \sum\limits_{\omega \in \mathcal{R}} \beta E \left \| h_\omega^t \right \|^2 \notag \\
    = & \frac{1}{R} \sum\limits_{\omega \in \mathcal{R}} (1 - \beta) E \left \| h_\omega^t \right \|^2 + \frac{1}{R} \sum\limits_{\omega \in \mathcal{R}} \beta E \left \| g_\omega^t \right \|^2. \notag
\end{align}
Then, applying Lemma \ref{lemma universal}, we have
\begin{align}
    & \hspace{-2em} E H^{t+1} \leq \frac{1}{R} \sum\limits_{\omega \in \mathcal{R}} (1 - \beta) E \left \| h_\omega^t \right \|^2 + \frac{1}{R} \sum\limits_{\omega \in \mathcal{R}} \beta E \left \| g_\omega^t \right \|^2  \\
    = & \frac{1}{R} \sum\limits_{\omega \in \mathcal{R}} (1 - \beta) E \left \| h_\omega^t \right \|^2 \notag \\
    & + \frac{1}{R} \sum\limits_{\omega \in \mathcal{R}} \beta E \left \| g_\omega^t - \nabla f_\omega(x^t) + \nabla f_\omega(x^t) \right \|^2 \notag \\
    = & \frac{1}{R} \sum\limits_{\omega \in \mathcal{R}} (1 - \beta) E \left \| h_\omega^t \right \|^2 \notag \\
    & + \frac{1}{R} \sum\limits_{\omega \in \mathcal{R}} \beta E \left \| g_\omega^t - \nabla f_\omega(x^t) \right \|^2 + \frac{1}{R} \sum\limits_{\omega \in \mathcal{R}} \beta \left \| \nabla f_\omega(x^t) \right \|^2 \notag \\
    \leq & (1-\beta)H^t + \beta L^2S^t + \beta \frac{1}{R} \sum\limits_{\omega \in \mathcal{R}} \left \| \nabla f_\omega(x^t) \right \|^2. \notag
\end{align}
Here the second equality uses $E [ g_\omega^t - \nabla f_\omega(x^t) ] = 0$ for any $\omega \in \mathcal{R}$. Since $f(x^t)$ has $L$-Lipschitz continuous gradients according to Assumption \ref{a1} and the outer variation is bounded according to Assumption \ref{a2}, we finally obtain
\begin{align}
    & E H^{t+1} \leq (1-\beta)H^t + \beta L^2S^t + \beta \frac{1}{R} \sum\limits_{\omega \in \mathcal{R}} \left \| \nabla f_\omega(x^t) \right \|^2 \notag \\
    = & (1-\beta)H^t + \beta L^2S^t \notag \\
    & + \beta \frac{1}{R} \sum\limits_{\omega \in \mathcal{R}} \left \| \nabla f_\omega(x^t) - \nabla f(x^t) + \nabla f(x^t) - \nabla f(x^*) \right \|^2 \notag
\end{align}
\begin{align}
    = & (1-\beta)H^t + \beta L^2S^t + \beta \frac{1}{R} \sum\limits_{\omega \in \mathcal{R}} \left \| \nabla f_\omega(x^t) - \nabla f(x^t) \right \| \notag \\
    & + \beta \frac{1}{R} \sum\limits_{\omega \in \mathcal{R}} \left \| \nabla f(x^t) - \nabla f(x^*) \right \|^2 \notag \\
    \leq & (1-\beta)H^t + \beta L^2S^t + \beta\sigma^2 + \beta L^2 \left \| x^t - x^* \right \|^2.
\end{align}
Here the first equality uses $\nabla f(x^*) = 0$ and the second equality uses $\sum_{\omega \in \mathcal{R}} ( \nabla f_\omega(x^t) - \nabla f(x^t) ) =0$. 
\end{proof}

Now we can provide the proof of Theorem \ref{theorem cd}.
\begin{proof}
We begin by manipulating $E \left \| x^{t+1} - x^* \right \|^2$ as
\begin{align}
\label{xt+1broad}
    & E \left \| x^{t+1} - x^* \right \|^2  \\
    = & E \left \| x^t - \gamma \nabla f(x^t) - x^* + x^{t+1} - x^t + \gamma \nabla f(x^t) \right \|^2 \notag \\
    \leq & \frac{1}{1 - \eta} \left \| x^t - \gamma\nabla f(x^t) - x^* \right \|^2  + \frac{1}{\eta} E \left \| x^{t+1} - x^t + \gamma\nabla f(x^t) \right \|^2, \notag
\end{align}
where $0< \eta <1$. The inequality comes from $\left \| a+b \right \|^2 \leq \frac{1}{1-\eta} \left \| a \right \|^2 + \frac{1}{\eta} \left \| b \right \|^2$.

Since $\nabla f(x^*) = 0$, we can bound the first term at the right-hand side of \eqref{xt+1broad} as
\begin{align}
\label{xtx*broad}
    & \left \| x^t - \gamma\nabla f(x^t) - x^* \right \|^2  \\
    = & \left \| x^t - \gamma \left ( \nabla f(x^t) - \nabla f(x^*) \right ) - x^* \right \|^2 \notag \\
    = & \left \| x^t -x ^* \right \|^2 - 2\gamma \langle \nabla f(x^t) - \nabla f(x^*), x^t - x^* \rangle \notag \\
    & + \gamma^2 \left \| \nabla f(x^t) - \nabla f(x^*) \right \|^2 \notag \\
    \leq & \left \| x^t - x^* \right \|^2 - 2\gamma\mu \left \| x^t -x^* \right \|^2 + \gamma^2 L^2 \left \| x^t -x^* \right \|^2 \notag \\
    = & (1 - 2\gamma\mu +\gamma^2L^2) \left \| x^t -x^* \right \|^2. \notag
\end{align}
Here we use $L \left \| x^t -x^* \right \| \geq \left \| \nabla f(x^t) - \nabla f(x^*) \right \|$ as $f(x)$ has Lipschitz continuous gradients and $\langle \nabla f(x^t) - \nabla f(x^*), x^t - x^* \rangle \geq \mu \left \| x^t - x^* \right \|^2$ as $f(x)$ is strongly convex.

Substituting (\ref{xtx*broad}) into (\ref{xt+1broad}), we can obtain
\begin{align}
\label{xt+1abroad}
    E \left \| x^{t+1} - x^* \right \|^2 \leq \frac{1 - 2\gamma\mu + \gamma^2L^2}{1-\eta} \left \| x^t - x^* \right \|^2  \\
    + \frac{1}{\eta} E \left \| x^{t+1} - x^t + \gamma\nabla f(x^t) \right \|^2. \notag
\end{align}
With $\eta = \frac{\gamma\mu}{2}$, if
\begin{equation}
\label{gamma1broad}
    \gamma^2L^2 \leq \frac{\gamma\mu}{2},
\end{equation}
it holds that
\begin{equation}
    \frac{1 - 2\gamma\mu + \gamma^2L^2}{1-\eta} \leq 1 - \gamma\mu.
\end{equation}
Therefore, (\ref{xt+1abroad}) can be rewritten as
\begin{align}
\label{xt+1finalbroad}
    E \left \| x^{t+1} - x^* \right \|^2 \leq (1 - \gamma\mu) \left \| x^t -x^* \right \|^2  \\
    + \frac{2}{\gamma\mu} E \left \| x^{t+1} - x^t + \gamma\nabla f(x^t) \right \|^2. \notag
\end{align}

We construct a \emph{Lyapunov function} $T^t$ as
\begin{equation}
\label{defLya}
    T^t := \left \| x^t -x^* \right \|^2 + cS^t + d\gamma^2H^t,
\end{equation}
where $c$ and $d$ are any positive constants, $S^t$ is defined as (\ref{defs1}), and $H^t$ is defined as (\ref{defH}). Thus, $T^t$ is non-negative.

Note that the Lyapunov function in \eqref{defLya} is different to the one in \eqref{defLyasaga} of Lemma \ref{lemma z*saga} for Byzantine-robust compressed SAGA, which is only parameterized by $c$. Below we will find proper $c$ and $d$.

Based on Lemma \ref{lemma universal} and (\ref{Ht+1}), it follows that
\begin{align}
\label{Tt+1a}
    & E T^{t+1}  =  E \left \| x^{t+1}  - x^* \right \|^2 + c E S^{t+1} + d \gamma^2 E H^{t+1}  \\
    \leq & (1 - \gamma\mu + 4cJ\gamma^2L^2 + d\beta \gamma^2 L^2) \left \| x^t - x^* \right \|^2 \notag \\
    & + \left ( \frac{2}{\gamma\mu} + 4cJ \right ) E \left \| x^{t+1} - x^t + \gamma\nabla f(x^t) \right \|^2 \notag \\
    & + \left ( (1 - \frac{1}{J^2})c + d\beta \gamma^2 L^2 \right ) S^t + d(1-\beta)\gamma^2 H^t + d\beta\gamma^2\sigma^2. \notag
\end{align}
Let $z_\epsilon^*$ be an $\epsilon$-approximate geometric median of $\{ \hat{g}_\omega^t, \omega \in \mathcal{W} \}$ and  observe the second term at the right-hand side of \eqref{Tt+1a}. We have
\begin{equation}
\hspace{-0.4em}    E \left \| x^{t+1} - x^t + \gamma\nabla f(x^t) \right \|^2 = \gamma^2 E \left \| z_\epsilon^* - \nabla f(x^t) \right \|^2.
\end{equation}
With this fact and Lemma \ref{lemma z*cd}, (\ref{Tt+1a}) can be rewritten as
\begin{align}
\label{Tcoeff}
    & E T^{t+1} \\
    \leq &  \{ 1 - \gamma\mu + 4cJ\gamma^2L^2 + \left ( \frac{2}{\gamma\mu} + 4cJ \right ) 4\delta C_\alpha^2 \gamma^2 L^2 \notag \\
    & + d\beta \gamma^2 L^2  \} \left \| x^t - x^* \right \|^2 + d\beta\gamma^2\sigma^2 \notag \\
    & + \{ (1 - \frac{1}{J^2})c + d\beta \gamma^2 L^2 \notag \\
    & + \left ( \frac{2}{\gamma\mu} + 4cJ \right )2(1 + \delta)C_\alpha^2 \gamma^2 L^2 \} S^t \notag \\
    & + \{ d(1 - \beta)\gamma^2 + \left ( \frac{2}{\gamma\mu} + 4cJ \right ) 4\delta C_\alpha^2 \gamma^2 \} H^t \notag \\
    & + \gamma^2 \left ( \frac{2}{\gamma\mu} + 4cJ \right ) \left ( 2C_\alpha^2 (1 + 2\delta)\sigma^2 + \frac{2\epsilon^2}{(W-2B)^2} \right ). \notag
\end{align}

If we let
\begin{equation}
\label{4cJ}
    4cJ\gamma^2L^2 \leq \frac{\gamma\mu}{17},
\end{equation}
then it holds
\begin{equation}
    \frac{2}{\gamma\mu} + 4cJ \leq \frac{2}{\gamma\mu} + \frac{\mu}{17\gamma L^2} \leq \frac{35}{17\gamma\mu}.
\end{equation}
If $\gamma$ and $d$ are chosen as $\frac{\gamma\mu}{2} \leq \frac{\beta}{2}$
and
\begin{equation}
\label{dvalue}
    d = \frac{35}{17} \cdot \frac{8\delta C_\alpha^2}{\beta\gamma\mu} \geq \frac{35}{17} \cdot \frac{4\delta C_\alpha^2}{\beta - \frac{\gamma\mu}{2}},
\end{equation}
Thus, the third coefficient at the right-hand side of (\ref{Tcoeff}) is bounded by
\begin{align}
\label{dleq}
    &  d(1 - \beta)\gamma^2 + \left ( \frac{2}{\gamma\mu} + 4cJ \right ) 4\delta C_\alpha^2 \gamma^2  \\
    \leq & d(1 - \beta)\gamma^2 + \frac{35}{17\gamma\mu} \cdot 4\delta C_\alpha^2\gamma^2 L^2 \notag \\
    \leq & (1 - \frac{\gamma\mu}{2})d\gamma^2. \notag
\end{align}

Similarly, if $\gamma$ and $c$ are chosen as $\frac{\gamma\mu}{2} \leq \frac{1}{2J^2}$
and
\begin{equation}
\label{cvalue}
\hspace{-0.6em}    c = \frac{35}{17} \frac{4(1 + 5\delta)C_\alpha^2J^2\gamma L^2}{\mu} \geq \frac{35}{17} \frac{2(1 + 5\delta)C_\alpha^2J^2\gamma L^2}{\mu(\frac{1}{J^2} - \frac{\gamma\mu}{2})},
\end{equation}
then with (\ref{dvalue}), the second coefficient at the right-hand side of (\ref{Tcoeff}) is bounded by
\begin{align}
\label{cleq}
    & (1 - \frac{1}{J^2})c + d\beta \gamma^2 L^2 + \left ( \frac{2}{\gamma\mu} + 4cJ \right )(4 - 2\delta)C_\alpha^2 \gamma^2 L^2 \notag \\
    \leq & (1 - \frac{1}{J^2})c + \frac{35}{17} \cdot \frac{8\delta C_\alpha^2\gamma L^2}{\mu} + \frac{35}{17} \cdot \frac{2(1+\delta)C_\alpha^2\gamma L^2}{\mu} \notag \\
    = & (1 - \frac{1}{J^2})c + \frac{35}{17} \cdot \frac{2(1+5\delta)C_\alpha^2\gamma L^2}{\mu} \leq (1 - \frac{\gamma\mu}{2})c.
\end{align}

Further, if $\delta C_\alpha^2 \leq \frac{\mu^2}{56L^2}$ and (\ref{4cJ}) is satisfied, then with (\ref{dvalue}), the first coefficient at the right-hand side of (\ref{Tcoeff}) is bounded
\begin{align}
\label{xleq}
    & 1 - \gamma\mu + 4cJ\gamma^2L^2 + \left ( \frac{2}{\gamma\mu} + 4cJ \right ) 4\delta C_\alpha^2 \gamma^2 L^2 + d\beta \gamma^2 L^2 \notag \\
    \leq & 1 - \gamma\mu + \frac{\gamma\mu}{17} + \frac{35}{17} \cdot \frac{4\delta C_\alpha^2\gamma L^2}{\mu} + \frac{35}{17} \cdot \frac{8\delta C_\alpha^2\gamma L^2}{\mu} \notag \\
    \leq & 1 - \gamma\mu + \frac{\gamma\mu}{17} + \frac{35}{17} \cdot \frac{12}{56} \cdot \gamma\mu = 1-\frac{\gamma\mu}{2}.
\end{align}

The last term at the right-hand side of (\ref{Tcoeff}) is bounded by
\begin{small}
\begin{align}
\label{deltaleq}
    & \gamma^2 \left ( \frac{2}{\gamma\mu} + 4cJ \right ) \left ( 2C_\alpha^2 (1 + 2\delta)\sigma^2 + \frac{2\epsilon^2}{(W-2B)^2} \right ) + d\beta\gamma^2\sigma^2 \notag \\
    \leq & \frac{35\gamma}{17\mu} \left ( 2C_\alpha^2 (1 + 2\delta)\sigma^2 + \frac{2\epsilon^2}{(W-2B)^2} \right ) + \frac{35\gamma}{17\mu} 8\delta C_\alpha^2\sigma^2 \notag \\
    = & \frac{35\gamma}{17\mu} \left ( 2C_\alpha^2 (1 + 6\delta)\sigma^2 + \frac{2\epsilon^2}{(W-2B)^2} \right ).
\end{align}
\end{small}

Therefore, substituting (\ref{dleq}), (\ref{cleq}), (\ref{xleq}), and (\ref{deltaleq}), we know that (\ref{Tcoeff}) satisfies
\begin{align}\label{thm3-0001}
    E T^{t+1} \leq \left ( 1 - \frac{\gamma\mu}{2} \right ) \left \| x^t - x^* \right \|^2 + \left ( 1 - \frac{\gamma\mu}{2} \right )c S^t  \\
    + \left ( 1 - \frac{\gamma\mu}{2} \right )d H^t + \frac{\gamma \mu}{2}\Delta_2, \notag
\end{align}
where
\begin{equation}
    \Delta_2 := \frac{70}{17\mu} \left ( 2C_\alpha^2 (1 + 6\delta)\sigma^2 + \frac{2\epsilon^2}{(W-2B)^2} \right ).
\end{equation}

Applying telescopic cancellation on \eqref{thm3-0001} from iteration $1$ to $t$ yields
\begin{equation}
    E T^t \leq \left ( 1 - \frac{\gamma\mu}{2} \right )^t \left ( T^0 - \frac{2}{\gamma\mu} \tilde{\Delta}_2 \right ) + \Delta_2.
\end{equation}
From the definition of Lyapunov function (\ref{defLya}), we obtain
\begin{equation}
    E \left \| x^t - x^* \right \|^2 \leq E T^t \leq \left ( 1 - \frac{\gamma\mu}{2} \right )^t \Delta_1 + \Delta_2,
\end{equation}
where
\begin{align}
    \Delta_1 :&= \left \| x^0 - x^* \right \|^2 - \Delta_2.
\end{align}

Considering all the requirements on the step size, we can choose $\gamma$ as
\begin{equation}
    \gamma \leq \frac{\beta\mu}{4\sqrt{35}\sqrt{1+5\delta} \cdot C_\alpha J^2L^2}.
\end{equation}
This completes the proof.
\end{proof}

\section{Analysis of Attacks-Free Compressed SGD}
\label{app-A}
The proof of Theorem \ref{theorem csgd} shares similarity with that in \cite{mishchenko2019distributed}, but the latter analyzes SGD with gradient difference compression, while ours considers general unbiased compressors. The work of \cite{alistarh2017qsgd} analyzes the convergence of quantized SGD with unbiased quantizer in terms of function value, while our performance metric is the distance to the optimal solution.

\begin{proof}
Based on the update rule (\ref{csgd}), it holds 
\begin{align}
    & E \left \| x^{t+1} - x^* \right \|^2 \notag \\
    = & E \left \| x^t - x^* - \frac{\gamma}{W}\sum\limits_{\omega \in \mathcal{W}} {\mathcal{Q}(\nabla f_{\omega,i_\omega^t}(x^t))} \right \|^2 \notag \\
    = & E \left \| x^t - x^* \right \|^2 - \frac{2\gamma}{W}\sum\limits_{\omega \in \mathcal{W}} E \left < {\mathcal{Q}(\nabla f_{\omega,i_\omega^t}(x^t))}, x^t - x^* \right > \notag \\
    & + E \left \| \frac{\gamma}{W}\sum\limits_{\omega \in \mathcal{W}} {\mathcal{Q}(\nabla f_{\omega,i_\omega^t}(x^t))} \right \|^2 \notag \\
    \leq & E \left \| x^t - x^* \right \|^2 - \frac{2\gamma}{W}\sum\limits_{\omega \in \mathcal{W}} E \left < {\mathcal{Q}(\nabla f_{\omega,i_\omega^t}(x^t))}, x^t - x^* \right > \notag \\
    & + \frac{\gamma^2}{W} \sum\limits_{\omega \in \mathcal{W}} E \left \| {\mathcal{Q}(\nabla f_{\omega,i_\omega^t}(x^t))} \right \|^2.
\end{align}
Since $\mathcal{Q}(\cdot)$ is an unbiased compressor and $\left \| {\mathcal{Q}(\nabla f_{\omega,i_\omega^t}(x^t))} \right \|^2$ $\leq (1+\delta) \left \| \nabla f_{\omega,i_\omega^t}(x^t) \right \|^2$, we can obtain
\begin{align}
\label{csgd1}
    & E \left \| x^{t+1} - x^* \right \|^2 \notag \\
    \leq & E \left \| x^t - x^* \right \|^2 - \frac{2\gamma}{W}\sum\limits_{\omega \in \mathcal{W}} E \left < {\nabla f_{\omega,i_\omega^t}(x^t)}, x^t - x^* \right > \notag \\
    & + \frac{\gamma^2 (1+\delta)}{W} \sum\limits_{\omega \in \mathcal{W}} E \left \| {\nabla f_{\omega,i_\omega^t}(x^t)} \right \|^2 \notag \\
    = & E \left \| x^t - x^* \right \|^2 - \frac{2\gamma}{W}\sum\limits_{\omega \in \mathcal{W}} \left < {\nabla f_{\omega}(x^t)}, x^t - x^* \right > \notag \\
    & + \frac{\gamma^2 (1+\delta)}{W} \sum\limits_{\omega \in \mathcal{W}} E \left \| {\nabla f_{\omega,i_\omega^t}(x^t)} \right \|^2 \notag \\
    = & E \left \| x^t - x^* \right \|^2 - 2\gamma \left < {\nabla f(x^t)}, x^t - x^* \right > \notag \\
    & + \frac{\gamma^2 (1+\delta)}{W} \sum\limits_{\omega \in \mathcal{W}} E \left \| {\nabla f_{\omega,i_\omega^t}(x^t)} \right \|^2,
\end{align}
where the first equality is from the unbiasness of stochastic gradient $\nabla f_{\omega,i_\omega^t}(x^t)$ and the second equality is from the fact that $\frac{1}{W}\sum_{\omega \in \mathcal{W}} \nabla f_\omega (x^t) = \nabla f(x^t)$.

With Assumptions \ref{a2} and \ref{a3}, the last term of (\ref{csgd1}) satisfies
\begin{align}\label{eq:thm1-001}
    & \frac{1}{W} \sum\limits_{\omega \in \mathcal{W}} E \left \| {\nabla f_{\omega,i_\omega^t}(x^t)} \right \|^2 \notag \\
    = & \frac{1}{W} \sum\limits_{\omega \in \mathcal{W}} E \left \| {\nabla f_{\omega,i_\omega^t}(x^t)} - \nabla f_\omega(x^t) + \nabla f_\omega(x^t) \right \|^2 \notag \\
    = & \frac{1}{W} \sum\limits_{\omega \in \mathcal{W}} \left \{ E \left \| {\nabla f_{\omega,i_\omega^t}(x^t)} - \nabla f_\omega(x^t) \right \|^2 + \left \| \nabla f_\omega(x^t) \right \|^2 \right \} \notag \\
    \leq & \zeta^2 + \frac{1}{W} \sum\limits_{\omega \in \mathcal{W}} \left \| \nabla f_\omega(x^t) - \nabla f(x^t) + \nabla f(x^t) \right \|^2 \notag \\
    = & \zeta^2 + \frac{1}{W} \sum\limits_{\omega \in \mathcal{W}} \left \| \nabla f_\omega(x^t) - \nabla f(x^t) \right \|^2 + \left \| \nabla f(x^t) \right \|^2 \notag \\
    \leq & \zeta^2 + \sigma^2 + \left \| \nabla f(x^t) \right \|^2,
\end{align}
where the last two equalities use $E ({\nabla f_{\omega,i_\omega^t}(x^t)} - \nabla f_\omega(x^t)) = 0$ and $\frac{1}{W} \sum_{\omega \in \mathcal{W}} (\nabla f_\omega(x^t) - \nabla f(x^t)) = 0$, respectively.

Using the inequality
\begin{align}
    \left < {\nabla f(x^t)} - \nabla f(x^*), x^t - x^* \right > \geq \frac{\mu L}{\mu + L} \left \| x^t - x^* \right \|^2 \notag \\
    + \frac{1}{\mu + L} \left \| \nabla f(x^t) - \nabla f(x^*) \right \|^2,
\end{align}
and $\nabla f(x^*) = 0$, with \eqref{eq:thm1-001} we can rewrite (\ref{csgd1}) as
\begin{align}
    & E \left \| x^{t+1} - x^* \right \|^2 \notag \\
    \leq & E \left \| x^t - x^* \right \|^2 - 2\gamma \left < {\nabla f(x^t)} - \nabla f(x^*), x^t - x^* \right > \notag \\
    & + \gamma^2 (1+\delta) \left (\zeta^2 + \sigma^2 + \left \| f(x^t) - \nabla f(x^*) \right \|^2 \right ) \notag \\
    \leq & \left (1 - \frac{2\gamma\mu L}{\mu + L}\right ) E \left \| x^t - x^* \right \|^2 \notag \\
    & + \left ( \gamma^2(1+\delta) - \frac{2\gamma}{\mu + L} \right ) \left \| f(x^t) - \nabla f(x^*) \right \|^2 \notag \\
    & + \gamma^2(1+\delta)(\zeta^2+\sigma^2).
\end{align}

If $\gamma$ satisfies
\begin{equation} \label{eq:thm1-002}
    \gamma \leq \frac{2}{(1+\delta)(\mu + L)},
\end{equation}
we have
\begin{align}
    E \left \| x^{t+1} - x^* \right \|^2 \leq \left (1 - \frac{2\gamma\mu L}{\mu + L}\right ) E \left \| x^t - x^* \right \|^2 \notag \\
    + \gamma^2(1+\delta)(\zeta^2+\sigma^2).
\end{align}

Applying telescopic cancellation from iteration 1 to $t$ yields
\begin{equation}
    E \left \| x^{t} - x^* \right \|^2 \leq \left (1 - \frac{2\gamma\mu L}{\mu + L}\right )^t \Delta_1 + \Delta_2',
\end{equation}
where
\begin{equation}
    \Delta_1 = \left \| x^0 - x^* \right \|^2 - \Delta_2',
\end{equation}
\begin{equation}
    \Delta_2' = \frac{\gamma(1+\delta)(\mu + L)(\zeta^2 + \sigma^2)}{2\mu L} \leq \frac{1}{\mu L}(\zeta^2 + \sigma^2).
\end{equation}
Here we use the upper bound of $\gamma$ in \eqref{eq:thm1-002}. 
\end{proof}

\section{Biased Compressors and Error Feedback}

Biased compressors, such as $\ell_1$-sign quantization and top-$k$ sparsification, are also widely used to improve communication efficiency of distributed algorithms. In this part, we will introduce Byzantine-robust and communication-efficient federated learning with biased compression, as well as error feedback, a corresponding compression noise reduction technique.

We first give the definition of a general compressor, which follows \cite{karimireddy2019error}.

\begin{definition}[General compressor]
\label{def biased}
A (possibly randomized) operator $\mathcal{Q}$: $\mathbb{R}^p \rightarrow \mathbb{R}^p$ is a \emph{general compressor} if it satisfies
\begin{equation}
    E_{\mathcal{Q}} \left \|\mathcal{Q}(x) - x\right \|^2 \le (1-\kappa)\left \|x\right \|^2, \quad \forall x \in \mathbb{R}^p.
\end{equation}
where $\kappa \in (0,1]$.
\end{definition}

A general compressor can be either unbiased or biased. Typical biased compressors include:
\begin{itemize}
    \item $\ell_1$-sign quantization: For any $x \in \mathbb{R}^p$, $\mathcal{Q}(x) = \frac{\left \|x \right \|_1}{p}{\rm sign}(x)$. Here $\kappa$ is $\frac{\left \|x \right \|_1^2}{p\left \|x \right \|}$.
    \item Top-$k$ sparsification: For any $x \in \mathbb{R}^p$, select $k$ elements with the largest absolute values to be remained, and let the other elements to be zero. Here $\kappa$ is $\frac{k}{p}$.
\end{itemize}

Like unbiased compressors that we focus on in the main text, biased compressors also introduce compression noise. An effective strategy to reduce compression noise for biased compressors is error feedback. The idea is to store the error between the compressed and original gradients and add it back to the gradient in the next iteration. It has been proved that error feedback can guarantee convergence of compressed stochastic algorithms and achieve gradient compression for free \cite{stich2018sparsified,karimireddy2019error,tang2019doublesqueeze}.

\subsection{Byzantine-Robust Compressed SAGA with Error Feedback}

The error feedback framework can be also combined with variance reduction to reduce both stochastic and compression noise. When applying error feedback to the Byzantine-robust compressed SAGA,  each regular worker $\omega \in \mathcal{R}$ at iteration $t$ computes the corrected local stochastic gradient $g_\omega^t$ and updates the accumulated error $e_\omega^{t+1}$ as
\begin{align}
    u_\omega^t =& g_\omega^t + e_\omega^t, \label{update ef-1} \\
    e_\omega^{t+1} =& u_\omega^t - \mathcal{Q}(u_\omega^t), \label{update ef-2}
\end{align}
where $e_\omega^t$ has been stored in the previous iteration. Each regular worker $\omega$ sends $\mathcal{Q}(u_\omega^t)$ to the master node. Then, the master node updates the model parameter as
\begin{equation}\label{update-saga-ef}
    x^{t+1} = x^t - \gamma \cdot \mathop{{\rm geomed}}\limits_{\omega \in \mathcal{W}} \{ \mathcal{Q}(u_w^t) \}.
\end{equation}

The Byzantine-robust compressed SAGA with error feedback is described in Algorithm \ref{algorithm ef}. In each regular worker $\omega$, a stochastic gradient table is kept to store the most recent stochastic gradient for every local sample, and an error vector $e_\omega^t$ is used to record the compression error. Each Byzantine worker $\omega$ may maintain its stochastic gradient table and $e_\omega^t$ for the sake of generating malicious vectors, or not do so but generate malicious vectors in other ways. At iteration $t$, the master node broadcasts $x^t$ to all the workers. Each regular worker $\omega$ randomly selects a sample and obtains the corrected stochastic gradient $g_\omega^t$ as \eqref{update g}. Next, the error vector $e_\omega^t$ is added to $g_\omega^t$ as in \eqref{update ef-1}, and the result is compressed and sent to the master node. Regular worker $\omega$ then updates $e_\omega^{t+1}$ according to \eqref{update ef-2}. The Byzantine workers can generate arbitrary messages but also send the compressed results to cheat the master node. After collecting the compressed messages from all the workers, the master node calculates the geometric median updates $x^{t+1}$ as in \eqref{update-saga-ef}.

\begin{algorithm}[tb]
\caption{Byzantine-Robust Compressed SAGA with Error Feedback}
\label{algorithm ef}
\textbf{Input}: Step size $\gamma$ \\
\textbf{Initialize}: Initialize $x^0$ for master node and all workers. Initialize $e_\omega^0 = 0$ for each worker $\omega$. Initialize
$\{ \nabla f_{\omega, j}(\phi_{\omega, j}^0) = \nabla f_{\omega,j}(x^0), j = 1,\dots,J \}$ for each regular worker $\omega$
\begin{algorithmic}[1] 
\FOR {$t=0,1,\dots$}
\STATE \textbf{Master node}:
\STATE Broadcast $x^t$ to all workers
\STATE Receive $\mathcal{Q}(u_\omega^t)$ from all workers
\STATE Update $x^{t+1} = x^t - \gamma \cdot \mathop{{\rm geomed}}_{\omega \in \mathcal{W}} \{ \mathcal{Q}(u_m^t) \}$
\STATE \textbf{Worker $\omega$}:
\IF {$\omega \in \mathcal{R}$}
\STATE Compute $\bar{g}_\omega^t = \frac{1}{J} \sum_{j=1}^J \nabla f_{\omega, j}(\phi_{\omega, j}^t)$
\STATE Randomly sample $i_\omega^t$ from $\{ 1,\dots,J \}$
\STATE Obtain $g_\omega^t = \nabla f_{\omega,i_\omega^t}(x^t) - \nabla f_{\omega,i_\omega^t}(\phi_{\omega, i_\omega^t}^t) + \bar{g}_\omega^t$
\STATE Store $\nabla f_{\omega,i_\omega^t}(\phi_{\omega, i_\omega^t}^t) = \nabla f_{\omega,i_\omega^t}(x^t)$
\STATE Compress $\mathcal{Q}(u_\omega^t) = \mathcal{Q}(g_\omega^t + e_\omega^t)$
\STATE Update $e_\omega^{t+1} = u_\omega^t - \mathcal{Q}(u_\omega^t)$
\STATE Send $\mathcal{Q}(u_\omega^t)$ to master node
\ELSIF {$\omega \in \mathcal{B}$}
\STATE Generate arbitrary malicious message $g_\omega^t = *$
\STATE Send $\mathcal{Q}(u_\omega^t) = \mathcal{Q}(g_\omega^t)$ to master node
\ENDIF
\ENDFOR
\end{algorithmic}
\end{algorithm}

\subsection{Theoretical Analysis}

Now we analyze Byzantine-robust SAGA with error feedback. Note that we use the geometric median aggregation in the master node instead of the mean aggregation, such that the perturbed iterate analysis in the existing works cannot be applied here. This makes the convergence analysis challenging.

We begin with reviewing a lemma that bounds the error vector $e_\omega^t$. The bound universally holds for stochastic algorithms as long as Assumption \ref{a4} holds.

\begin{lemma} \cite{karimireddy2019error}
\label{lemma bounded error}
Consider Algorithm \ref{algorithm ef} with $\epsilon$-approximate geometric median aggregation and using a general compressor. Under Assumption \ref{a4}, for any regular worker $\omega \in \mathcal{R}$ and at any iteration $t$, the error vector $e_\omega^t$ is bounded as
\begin{equation}
    E \left \| e_\omega^t \right \|^2 \leq \frac{4(1 - \kappa)}{\kappa^2} G^2, \quad \forall t \geq 0.
\end{equation}
\end{lemma}

To handle the bias introduced by geometric median, we give the following lemma to describe the gap between the $\epsilon$-approximate geometric median of $\{ \mathcal{Q}(u_\omega^t), \omega \in \mathcal{W} \}$ and $\nabla f(x^t)$.

\begin{lemma}
\label{lemma z*ef}
Consider Algorithm \ref{algorithm ef} with $\epsilon$-approximate geometric median aggregation and using a general compressor. Under Assumptions \ref{a1}, \ref{a2}, and \ref{a4}, if the number of Byzantine workers satisfies $B < \frac{W}{2}$, then an $\epsilon$-approximate geometric median of $\{ \mathcal{Q}(u_\omega^t), \omega \in \mathcal{W} \}$, denoted as $z_\epsilon^*$, satisfies
\begin{equation}
    E \left \| z_\epsilon^* - \nabla f(x^t) \right \|^2 \leq 4C_\alpha^2L^2S^t + 4C_\alpha^2\sigma^2 + \frac{64C_\alpha^2(1-\kappa)}{\kappa^2} G^2 + \frac{2\epsilon^2}{(W-2B)^2}.
\end{equation}
\end{lemma}

\begin{proof}
From (\ref{update ef-2}), we have
\begin{align}
    \mathcal{Q}(u_\omega^t) = u_\omega^t - e_\omega^{t+1} = g_\omega^t + e_\omega^t - e_\omega^{t+1}.
\end{align}
According to Lemma \ref{lemma geomed}, it holds that
\begin{align}
    & E \left \| z_\epsilon^* - \nabla f(x^t) \right \|^2 \\
    \leq & 2C_\alpha^2 \frac{1}{R} \sum\limits_{\omega \in \mathcal{R}} E \left \| \mathcal{Q}(u_\omega^t) - \nabla f(x^t) \right \|^2 + \frac{2\epsilon^2}{(W-2B)^2} \notag \\
    = & 2C_\alpha^2 \frac{1}{R} \sum\limits_{\omega \in \mathcal{R}} E \left \| g_\omega^t + e_\omega^t - e_\omega^{t+1} - \nabla f(x^t) \right \|^2 + \frac{2\epsilon^2}{(W-2B)^2} \notag \\
\leq & 2C_\alpha^2 \frac{1}{R} \sum\limits_{\omega \in \mathcal{R}} \left \{ 2 E \left \| g_\omega^t - \nabla f(x^t) \right \|^2 + 2 E \left \| e_\omega^t - e_\omega^{t+1} \right \|^2 \right \} + \frac{2\epsilon^2}{(W-2B)^2} \notag \\
    \leq & 2C_\alpha^2 \frac{1}{R} \sum\limits_{\omega \in \mathcal{R}} \left \{ 2 E \left \| g_\omega^t - \nabla f_\omega(x^t) + \nabla f_\omega(x^t) - \nabla f(x^t) \right \|^2 + 4E \left \| e_\omega^t \right \|^2 + 4E \left \| e_\omega^{t+1} \right \|^2 \right \} + \frac{2\epsilon^2}{(W-2B)^2}. \notag
\end{align}
Here the last two inequalities come from the fact that $\left \|a+b \right \|^2 \leq 2 \left \| a \right \|^2 + 2 \left \| b \right \|^2$. Applying Lemma \ref{lemma universal}, Lemma \ref{lemma bounded error}, and Assumption \ref{a2}, we have
\begin{align}
    & E \left \| z_\epsilon^* - \nabla f(x^t) \right \|^2 \\
    \leq & 2C_\alpha^2 \frac{1}{R} \sum\limits_{\omega \in \mathcal{R}} \left \{ 2 E \left \| g_\omega^t - \nabla f_\omega(x^t) \right \|^2 + 2E \left \| \nabla f_\omega(x^t) - \nabla f(x^t) \right \|^2 + \frac{32(1-\kappa)}{\kappa^2} G^2 \right \} + \frac{2\epsilon^2}{(W-2B)^2} \notag \\
    \leq & 4C_\alpha^2L^2S^t + 4C_\alpha^2\sigma^2 + \frac{64C_\alpha^2(1-\kappa)}{\kappa^2} G^2 + \frac{2\epsilon^2}{(W-2B)^2}. \notag
\end{align}
Here the first inequality uses $\sum_{\omega \in \mathcal{R}} ( \nabla f_\omega(x^t) - \nabla f(x^t) ) =0$. This completes the proof.
\end{proof}

Now we establish the convergence of the Byzantine-robust compressed SAGA with error feedback.
\begin{theorem}[Convergence of Byzantine-robust compressed SAGA with error feedback]
\label{theorem ef}
Consider Algorithm \ref{algorithm ef} with $\epsilon$-approximate geometric median aggregation and using a general compressor. Under Assumptions \ref{a1}, \ref{a2}, and \ref{a4}, if the number of Byzantine workers satisfies $B < \frac{W}{2}$, and the step size $\gamma$ satisfies
\begin{equation}
    \gamma \leq \frac{\mu}{4\sqrt{10}J^2L^2C_\alpha},
\end{equation}
then it holds that
\begin{equation}
    E \left \| x^t - x^* \right \|^2 \leq \left ( 1 - \frac{\gamma\mu}{2} \right )^t \Delta_1 + \Delta_2,
\end{equation}
where
\begin{align}
    \Delta_1 :=& \left \| x^0 - x^* \right \|^2 - \Delta_2, \\
    \Delta_2 :=& \frac{5}{\mu^2} \left ( 4C_\alpha^2\sigma^2 + \frac{64C_\alpha^2(1-\kappa)}{\kappa^2} G^2 + \frac{2\epsilon^2}{(W-2B)^2} \right ).
\end{align}
\end{theorem}

\begin{proof}
When the step size $\gamma$ is sufficiently small such that
\begin{equation}
\label{gamma1ef}
    \gamma^2L^2 \leq \frac{\gamma\mu}{2},
\end{equation}
we can observe that \eqref{xt+1final} also holds true here.

To prove the theorem, we construct a \emph{Lyapunov function} $T^t$ as
\begin{equation}
\label{defLyaef}
    T^t := \left \| x^t -x^* \right \|^2 + cS^t,
\end{equation}
where $c$ is any positive constant and $S^t$ is defined as (\ref{defs}). Thus, $T^t$ is non-negative.

Based on (\ref{St+1}), it follows that
\begin{align}
\label{Tt+1aef}
    E T^{t+1} & = E \left \| x^{t+1}  - x^* \right \|^2 + c E S^{t+1} \\
    \leq & (1 - \gamma\mu + 4cJ\gamma^2L^2) \left \| x^t - x^* \right \|^2 + \left ( \frac{2}{\gamma\mu} + 4cJ \right ) E \left \| x^{t+1} - x^t + \gamma\nabla f(x^t) \right \|^2 + \left ( 1 - \frac{1}{J^2} \right )c S^t. \notag
\end{align}
Let $z_\epsilon^*$ be an $\epsilon$-approximate geometric median of $\{ \mathcal{Q}(v_\omega^t), \omega \in \mathcal{W} \}$ and observe the second term at the right-hand side of \eqref{Tt+1aef}. We have
\begin{equation}
    E \left \| x^{t+1} - x^t + \gamma\nabla f(x^t) \right \|^2 = \gamma^2 E \left \| z_\epsilon^* - \nabla f(x^t) \right \|^2.
\end{equation}
With this fact and Lemma \ref{lemma z*ef}, (\ref{Tt+1aef}) can be rewritten as
\begin{align}
\label{Tcoeffef}
    E T^{t+1} \leq & \left [ 1 - \gamma\mu + 4cJ\gamma^2L^2 \right ] \left \| x^t - x^* \right \|^2 \\
    & + \left [ (1 - \frac{1}{J^2})c + \left ( \frac{2}{\gamma\mu} + 4cJ \right )4C_\alpha^2 \gamma^2 L^2 \right ] S^t \notag \\
    & + \gamma^2 \left ( \frac{2}{\gamma\mu} + 4cJ \right ) \left ( 4C_\alpha^2\sigma^2 + \frac{64C_\alpha^2(1-\kappa)}{\kappa^2} G^2 + \frac{2\epsilon^2}{(W-2B)^2} \right ). \notag
\end{align}

If we let
\begin{equation}
\label{4cJef}
    4cJ\gamma^2L^2 \leq \frac{\gamma\mu}{2},
\end{equation}
then it holds
\begin{equation}
    \frac{2}{\gamma\mu} + 4cJ \leq \frac{2}{\gamma\mu} + \frac{\mu}{2\gamma L^2} \leq \frac{5}{2\gamma\mu}.
\end{equation}
Thus, the first coefficient at the right-hand side of \eqref{Tcoeffef} is bounded by
\begin{align}
\label{xleqef}
    1 - \gamma\mu + 4cJ\gamma^2L^2 \leq 1-\frac{\gamma\mu}{2}.
\end{align}

Similarly, if $\gamma$ and $c$ are chosen as
\begin{equation}
\label{gamma2ef}
    \frac{\gamma\mu}{2} \leq \frac{1}{2J^2},
\end{equation}
and
\begin{equation}
\label{cvalueef}
    c = \frac{20C_\alpha^2J^2\gamma L^2}{\mu} \geq \frac{5}{2} \cdot \frac{4C_\alpha^2J^2\gamma L^2}{\mu(\frac{1}{J^2} - \frac{\gamma\mu}{2})},
\end{equation}
then the second coefficient at the right-hand side of (\ref{Tcoeffef}) is bounded by
\begin{align}
\label{cleqef}
    & (1 - \frac{1}{J^2})c + \left ( \frac{2}{\gamma\mu} + 4cJ \right )4C_\alpha^2 \gamma^2 L^2 \\
    \leq & (1 - \frac{1}{J^2})c + \frac{5}{2} \cdot \frac{4C_\alpha^2\gamma L^2}{\mu} \notag \\
    \leq & (1 - \frac{\gamma\mu}{2})c. \notag
\end{align}

The last term at the right-hand side of (\ref{Tcoeffef}) is bounded by
\begin{align}
\label{deltaleqef}
    & \gamma^2 \left ( \frac{2}{\gamma\mu} + 4cJ \right ) \left ( 4C_\alpha^2\sigma^2 + \frac{64C_\alpha^2(1-\kappa)}{\kappa^2} G^2 + \frac{2\epsilon^2}{(W-2B)^2} \right ) \\
    \leq & \frac{5\gamma}{2\mu} \left ( 4C_\alpha^2\sigma^2 + \frac{64C_\alpha^2(1-\kappa)}{\kappa^2} G^2 + \frac{2\epsilon^2}{(W-2B)^2} \right ). \notag
\end{align}

Therefore, substituting (\ref{xleqef}), (\ref{cleqef}), and (\ref{deltaleqef}), we know that (\ref{Tcoeffef}) satisfies
\begin{equation}\label{thm4-000}
    E T^{t+1} \leq \left ( 1 - \frac{\gamma\mu}{2} \right ) \left \| x^t - x^* \right \|^2 + \left ( 1 - \frac{\gamma\mu}{2} \right )c S^t + \frac{\gamma \mu}{2} \Delta_2,
\end{equation}
where
\begin{equation}
    \Delta_2 := \frac{5}{\mu^2} \left ( 4C_\alpha^2\sigma^2 + \frac{64C_\alpha^2(1-\kappa)}{\kappa^2} G^2 + \frac{2\epsilon^2}{(W-2B)^2} \right ).
\end{equation}

Applying telescopic cancellation on \eqref{thm4-000} from iteration $1$ to $t$ yields
\begin{equation}
    E T^t \leq \left ( 1 - \frac{\gamma\mu}{2} \right )^t \left ( T^0 - \Delta_2 \right ) + \frac{2}{\gamma\mu} \tilde{\Delta}_2.
\end{equation}

From the definition of Lyapunov function (\ref{defLyaef}), we can obtain
\begin{equation}
    E \left \| x^t - x^* \right \|^2 \leq E T^t \leq \left ( 1 - \frac{\gamma\mu}{2} \right )^t \Delta_1 + \Delta_2,
\end{equation}
where
\begin{align}
    \Delta_1 := \left \| x^0 - x^* \right \|^2 - \Delta_2.
\end{align}

Considering the requirements (\ref{gamma1ef}), (\ref{4cJ}), and (\ref{gamma2ef}) on the step size, we can choose $\gamma$ as
\begin{equation}
    \gamma \leq \frac{\mu}{4\sqrt{10} \cdot C_\alpha J^2L^2}.
\end{equation}
This completes the proof.
\end{proof}

Theorem \ref{theorem ef} shows that the Byzantine-robust compressed SAGA with error feedback can also linearly converge to a neighborhood of the optimal solution. However, the analysis needs the stochastic gradients to be bounded, which is common in the analysis of error feedback. The learning error $\Delta_2$ is linear with $G^2$ and can be very large.
Improving the proof techniques and obtain a tighter bound of learning error for error feedback will be our future work.

\subsection{Numerical Experiments}

Here we provide numerical experiments of Byzantine-robust compressed SAGA with error feedback to illustrate its effectiveness. The considered problem is also logistic regression. The dataset and detailed settings are the same as those in the main text.

\begin{figure*}[htb]
    \centering
    \includegraphics[width=0.90\textwidth]{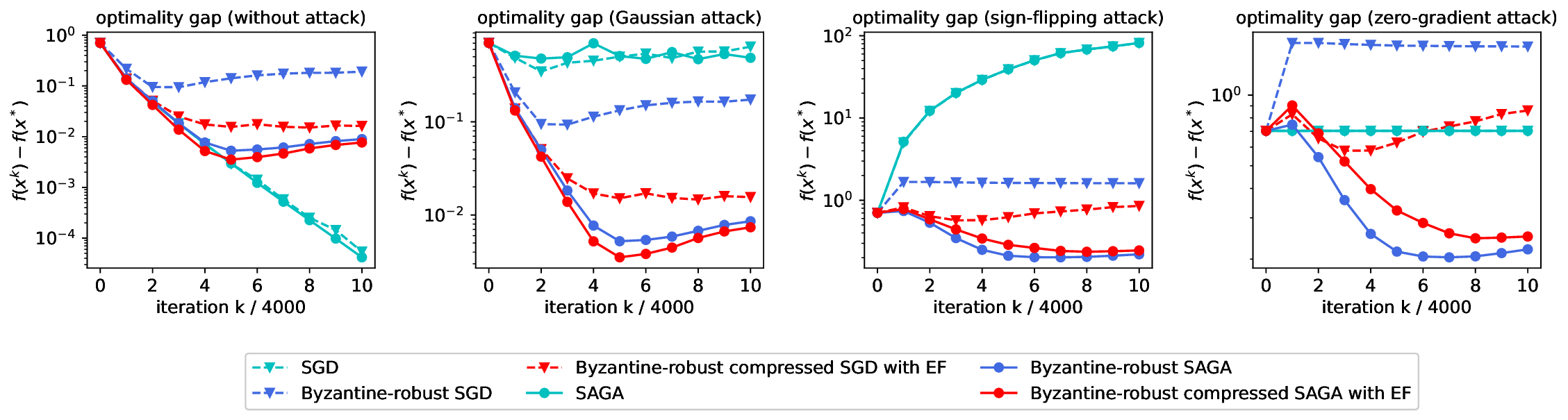}
    \caption{Effect of reduction of stochastic and compression noise.}
    \label{compressionef}
\end{figure*}

\begin{figure*}[htb]
    \centering
    \includegraphics[width=0.90\textwidth]{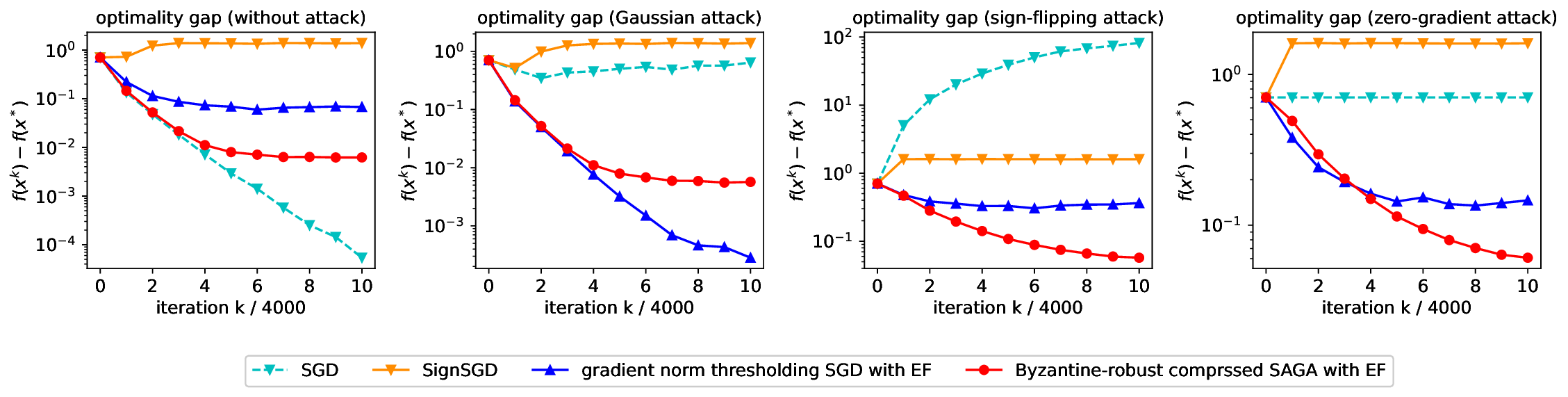}
    \caption{Comparison of proposed algorithm and existing methods: SignSGD and gradient norm thresholding SGD.}
    \label{comparisonef}
\end{figure*}

Figure \ref{compressionef} depicts the optimality gap $f(x^t) - f(x^*)$ of SGD, Byzantine-robust SGD, Byzantine-robust compressed SGD with error feedback (EF), SAGA, Byzantine-robust SAGA, and Byzantine-robust compressed SAGA with EF. The compressor here is top-$k$ spasification and ratio $k/p$ is 0.1. The Byzantine workers also obey the top-$k$ sparsification compression rule and the error feedback framework, so as to recover the effects of uncompressed Byzantine attacks as much as possible. Observe that the Byzantine-robust compressed SAGA with EF has the ability to defend all the three attacks as the Byzantine-robust SAGA without compression, and their learning errors are similar. This fact implies that error feedback can successfully reduce compression noise for a biased compressor and achieve compression for free, too.

Figure \ref{comparisonef} compares the Byzantine-robust SAGA with EF with SignSGD and the gradient norm thresholding SGD. Here the compressor is $\ell_1$-sign quantization, and the gradient norm thresholding SGD also uses error feedback as in \cite{ghosh2021communication}. The fraction of removed compressed gradients is 0.3. When only Gaussian attacks are considered, the gradient norm thresholding SGD with EF behaves the best because all the malicious messages are removed, such that the training process is similar to that of SGD without attacks. However, it cannot remove all the malicious messages under the sign-flipping and zero-gradient attacks and behaves worse than the Byzantine-robust SAGA with EF. On the contrary, the Byzantine-robust SAGA with EF is able to defend various Byzantine attacks.

\end{document}